%% file: paper.tex
\documentclass[twoside,11pt]{article}
\usepackage{jmlr2e}
\usepackage{graphicx}
\usepackage{amsmath,amssymb}
\usepackage{algorithm,algorithmic}
\usepackage{MnSymbol}
\usepackage{todonotes} 
\usepackage{url}
\usepackage[colorlinks=true]{hyper ref}

\DeclareMathOperator*{\Ex}{\vphantom{p}\mathbb{E}}
\let\inf\undef
\DeclareMathOperator*{\inf}{\vphantom{p}inf}
\let\sup\undef
\DeclareMathOperator*{\sup}{\vphantom{p}sup}

\jmlrheading{1}{2000}{1-48}{4/00}{10/00}{Rakhlin, Sridharan, Tewari}

\begin{document}

\title{Online Learning via Sequential Complexities}

\author{\name Alexander Rakhlin \email rakhlin@wharton.upenn.edu \\
       \addr Department of Statistics\\
       University of Pennsylvania
       \AND
       \name Karthik Sridharan \email skarthik@wharton.upenn.edu \\
       \addr Department of Statistics\\
       University of Pennsylvania
       \AND
       \name Ambuj Tewari \email tewaria@umich.edu \\
       \addr Department of Statistics, and\\
       Department of Electrical Engineering and Computer Science \\
       University of Michigan
}

\editor{Mehryar Mohri}

\maketitle

\begin{abstract}
We consider the problem of sequential prediction and provide tools to study the minimax value of the associated game. Classical statistical
learning theory provides several useful complexity measures to study learning with i.i.d.\ data. Our proposed sequential complexities can be seen as extensions of these measures to the sequential setting. The developed theory is shown to yield precise learning guarantees for the problem of sequential prediction. In particular, we show necessary and sufficient conditions for online learnability in the setting of supervised learning. Several examples show the utility of our framework: we can establish learnability without having to exhibit an explicit online learning algorithm.
\end{abstract}

\newcommand \grey[1]{{\color[rgb]{0.80,0.80,0.80}#1}}
\newcommand \red[1]{{\color[rgb]{0.80,0.00,0.00}\textbf{#1}}}
\newcommand{\multiminimax}[1]{\ensuremath{\left\llangle #1\right\rrangle}}


\newcommand{\Treeproc}{\ensuremath{\mathbb T}}
\newcommand{\Mproc}{\ensuremath{\mathbb M}}
\newcommand{\Emproc}{\ensuremath{\mathbb G}}
\newcommand{\Radproc}{\ensuremath{ \mathbb S}}

\newcommand{\mbb}[1]{\mathbb{#1}}
\newcommand{\mbf}[1]{\mathbf{#1}}
\newcommand{\mc}[1]{\mathcal{#1}}
\newcommand{\mrm}[1]{\mathrm{#1}}
\newcommand{\trm}[1]{\textrm{#1}}

\newcommand{\norm}[1]{\left\|#1\right\|}
\newcommand{\ip}[2]{\left<#1,#2\right>}
\newcommand{\sign}{\mrm{sign}}
\newcommand{\argmin}[1]{\underset{#1}{\mrm{argmin}} \ }
\newcommand{\argmax}[1]{\underset{#1}{\mrm{argmax}} \ }
\newcommand{\reals}{\mathbb{R}}
\newcommand{\E}[1]{\mathbb{E}\left[ #1 \right]} 
\newcommand{\Ebr}[1]{\mathbb{E}\left\{ #1 \right\}} 
\newcommand{\En}{\mathbb{E}}  
\newcommand{\Es}[2]{\mathbb{E}_{#1}\left[ #2 \right]} 
\newcommand{\Ps}[2]{\mathbb{P}_{#1}\left[ #2 \right]}
\newcommand{\Prob}{\mathbb{P}}
\newcommand{\conv}{\operatorname{conv}}
\newcommand{\inner}[1]{\left\langle #1 \right\rangle}
\newcommand{\lv}{\left\|}
\newcommand{\rv}{\right\|}
\newcommand{\Phifunc}[1]{\Phi\left(#1\right)}
\newcommand{\phifunc}[1]{\phi\left(#1\right)}
\newcommand{\ind}[1]{{\bf 1}\left\{#1\right\}}
\newcommand{\tr}{\ensuremath{{\scriptscriptstyle\mathsf{T}}}}
\newcommand{\eqdist}{\stackrel{\text{d}}{=}}
\newcommand{\alphT}{\widehat{\alpha}(T)}
\newcommand{\PD}{\mathcal P}
\newcommand{\QD}{\mathcal Q}
\newcommand{\jp}{\ensuremath{\mathbf{p}}}
\newcommand{\loss}{\ensuremath{\mathbf{\ell}}}

\newcommand\s{\mathbf{s}}
\newcommand\w{\mathbf{w}}
\newcommand\x{\mathbf{x}}
\newcommand\y{\mathbf{y}}
\newcommand\z{\mathbf{z}}

\def\deq{\triangleq}

\renewcommand\v{\mathbf{v}}

\newcommand\cA{\mathcal{A}}
\newcommand\cB{\mathcal{B}}
\newcommand\cC{\mathcal{C}}
\newcommand\cD{\mathcal{D}}
\newcommand\X{\mathcal{X}}
\newcommand\Y{\mathcal{Y}}
\newcommand\Z{\mathcal{Z}}
\newcommand\F{\mathcal{F}}
\newcommand\G{\mathcal{G}}
\newcommand\cH{\mathcal{H}}
\newcommand\N{\mathcal{N}}
\newcommand\M{\mathcal{M}}
\newcommand\W{\mathcal{W}}
\newcommand\Nhat{\mathcal{\widehat{N}}}

\newcommand\ldim{\mathrm{Ldim}}
\newcommand\fat{\mathrm{fat}}
\newcommand\Img{\mbox{Img}}
\newcommand\sparam{\sigma} 
\newcommand\Psimax{\ensuremath{\Psi_{\mathrm{max}}}}

\newcommand\Rad{\mathfrak{R}}
\newcommand\Val{\mathcal{V}}
\newcommand\Valdet{\mathcal{V}^{\mathrm{det}}}
\newcommand\Dudley{\mathfrak{D}}
\newcommand\Reg{\mbf{R}}
\newcommand\D{\mbf{D}}
\renewcommand\P{\mbf{P}}

\newcommand\Zcvx{\Z_\mathrm{cvx}}
\newcommand\Zlin{\Z_\mathrm{lin}}

\input{intro}

\input{complexity}

\input{seqcomplex}

\input{rademacher}

\input{supervised}

\input{examples}

\input{discussion}

\acks{We would like to thank J. Michael Steele and Dean Foster for helpful discussions. We gratefully acknowledge the support of NSF under grants CAREER DMS-0954737 and CCF-1116928.}

\appendix
\input{minimax}
\input{appendix}

\bibliography{paper}

\end{document}

%% file: intro.tex
\section{Introduction}

This paper is concerned with sequential prediction problems where no probabilistic assumptions are made regarding the data generating mechanism.
Our viewpoint is expressed well by the following quotation from \cite{CovShe77}:
\begin{quotation}
``We are interested in sequential prediction procedures that exploit any apparent order in the sequence. We do not assume the existence of any
underlying distributions, but assume that the sequence is an outcome of a game against a malevolent intelligent nature.''
\end{quotation}
We will, in fact, take the game theoretic viewpoint seriously. All our investigations will proceed by analyzing the minimax value of a repeated game
between a \emph{player} or \emph{learner} and a ``malevolent intelligent nature", or the \emph{adversary}.

Even though we have the setting of prediction problems in mind, it will be useful to develop the theory in a somewhat abstract setting. Towards this end,
fix the sets $\F$ and $\Z$, as well as a loss function $\loss:\F\times\Z\to\reals$, and consider the following $T$-round repeated two-player game, which we term the \emph{online learning} or \emph{sequential prediction} model. On round $t \in \{1,\ldots,T\}$, the learner chooses $f_t \in \F$, 
the adversary picks $z_t \in \Z$, and the learner suffers loss $\loss(f_t,z_t)$. At the end of $T$ rounds we define \emph{regret}
$$
\Reg(f_{1:T},z_{1:T}) ~\deq~ \sum_{t=1}^T \loss(f_t,z_t) - \inf_{f \in \F} \sum_{t=1}^T \loss(f,z_t)
$$
as the difference between the cumulative loss of the player and the cumulative loss of the best fixed decision. For the given pair $(\F, \Z)$, the problem is said to be {\em online learnable} if there exists an algorithm for the learner such that regret grows sublinearly in the time horizon $T$, no matter what strategy the adversary employs.

The origin of the online learning (or sequential prediction) model can be traced back to the work of \cite{Robbins1951} on compound statistical decision problems.
Some of the earliest sequential prediction algorithms were proposed by \cite{Blackwell1956Analog,Blackwell1956Controlled} and \cite{Hannan57}. Blackwell's method was based on his celebrated approachability theorem whereas Hannan's was based on minimizing a randomly perturbed sum of previous losses. Hannan's ideas were to later
resurface in the influential Follow-the-Perturbed-Leader family \citep{Kalai2005} of online learning algorithms. The seminal ideas in the work of Robbins, Blackwell and Hannan led to further developments in many different fields. \cite{Cover1967}, \cite{Davisson1973}, \cite{Ziv1977}, \cite{Rissanen1984}, \cite{Feder1992}, and others laid the foundation of universal coding, compression and prediction in the Information Theory literature. Within Computer Science, \cite{Littlestone1994}, \cite{Cesa-Bianchi1997}, \cite{Vovk1998}, and others studied the online learning model and the prediction with expert advice framework. The connections between regret minimization and convergence to equilibria was studied in Economics by \cite{Foster1997}, \cite{Hart2000} and others.

We have no doubt left out many interesting works above. But even our partial list will convince the reader that research in online learning and sequential prediction has benefited from contributions by researchers from a variety of fields including Computer Science, Economics, Information Theory, and Statistics. For an excellent synthesis and presentation of results from these different fields we refer the reader to the book by \citet{PLG}.
Many of the ideas in the field are \emph{constructive}, resulting in beautiful algorithms, or
algorithmic techniques, associated with names such as Follow-the-Regularized-Leader, Follow-the-Perturbed-Leader, Weighted Majority, Hedge, and Online Gradient Descent. However, analyzing specific algorithms has obvious disadvantages. The algorithm may not be ``optimal" for the task at hand. Even if it is optimal, one cannot prove that fact unless one develops tools for analyzing the inherent \emph{complexity} of the online learning problem.

Our goal is precisely to provide such tools. We will begin by defining the minimax value of the game underlying the abstract online learning model. Then we will develop tools for controlling the minimax value resulting in a theory that parallels statistical learning theory. In particular, we develop analogues
of combinatorial dimensions, covering numbers, and Rademacher complexities. We will also provide results relating these complexities.

Note that our approach is \emph{non-constructive}: controlling the sequential complexities mentioned above will only guarantee the \emph{existence} of a good online learning algorithm but will not explicitly create one. However, it turns out that that the minimax point of view can indeed lead to constructive algorithms as shown by \cite{Rakhlin2012Relax}.

%% file: complexity.tex
\section{Minimax Value and Online Learnability}
\label{sec:value}

To proceed further in our analysis of the minimax value of the repeated game between the learner and the adversary, we need to make a few technical assumptions. We assume that $\F$ is a subset of a separable metric space. Let $\QD$ be the set of probability measures on $\F$ and assume that $\QD$ is weakly compact. In order to allow randomized prediction, we allow the learner to choose a distribution $q_t\in\QD$ on every round. The minimax value of the game is then defined as
\begin{align}  
	\label{eq:def_val_game}
	\Val_T(\F, \Z)  ~\deq~ \inf_{q_1\in \QD}\sup_{z_1\in\Z} \Ex_{f_1 \sim q_1} \cdots  \inf_{q_T\in \QD}\sup_{z_T\in\Z} \Ex_{f_T \sim q_T} \left[ \sum_{t=1}^T \loss(f_t,z_t) - \inf_{f\in \F}\sum_{t=1}^T \loss(f,z_t)\right] \ .
\end{align}
Henceforth, the notation $\Ex_{f\sim q}$ stands for the expectation operator integrating out the random variable $f$ with distribution $q$. We consider here the {\em adaptive} adversary who gets to choose each $z_t$ based on the history of moves $f_{1:t-1}$ and $z_{1:t-1}$. 

The first key step in the study of the value of the game is to appeal to the minimax theorem and exchange the pairs of infima and suprema in \eqref{eq:def_val_game}. This dual formulation is easier to analyze because the choice of the player comes \emph{after} the choice of the mixed strategy of the adversary. We remark that the minimax theorem holds under a very general assumption of weak compactness of $\QD$ and lower semi-continuity of the loss function.\footnote{
We refer to Appendix~\ref{sec:minimax_appendix} for a precise statement of the minimax theorem, as well as sufficient conditions.
} Under these conditions, we can appeal to Theorem~\ref{thm:minimax} stated below, which is adapted for our needs from the work of \citet{AbeAgaBarRak09}. 

\begin{theorem}\label{thm:minimax}
	Let $\F$ and $\Z$ be the sets of moves for the two players, satisfying the necessary conditions for the minimax theorem to hold.  Denote by $\mc{Q}$ and $\mc{P}$ the sets of probability measures (mixed strategies)  on $\F$ and $\Z$, respectively. Then
\begin{align}
	\Val_T(\F, \Z) 
	&=\sup_{p_1} \Ex_{z_1\sim p_1} \cdots~ \sup_{p_T} \Ex_{z_T\sim p_T} \left[
	  \sum_{t=1}^T \inf_{f_t \in \F}
	  	\Ex_{z_t \sim p_t} \left[ \loss(f_t,z_t) \right] - \inf_{f\in\F} \sum_{t=1}^T \loss(f,z_t)
	\right], \label{eq:value_equality}
\end{align}	
where suprema over $p_t$ range over all distributions in $\mc{P}$.
\end{theorem}
The question of learnability in the online learning model is now reduced to the study of $\Val_T(\F,\Z)$, taking Eq.~\eqref{eq:value_equality} as the starting point. 

\begin{definition}
	\label{def:learnability}
	A class $\F$ is said to be {\em online learnable} with respect to the given $\Z$ and $\ell$ if
	$$ \limsup_{T\to \infty} \frac{\Val_T(\F, \Z)}{T} \le 0 \ .$$
\end{definition}

Note that our notion of learnability is related to, but distinct from, {\em Hannan consistency} \citep{Hannan57, PLG}. The latter notion requires the iterated game to go on for an infinite number of rounds and is formulated in terms of \emph{almost sure convergence}. In contrast, we consider a distinct game for each $T$ and look at {\em expected} regret. Nevertheless, it is possible to obtain Hannan consistency using the techniques developed in this paper by considering a slightly different game \citep{RakSriTew11}. 

We also remark that the statements in this paper extend to the case when the learner is allowed to make decisions in a larger set $\G$, while the best-in-hindsight term in the regret definition is computed with respect to $\F\subseteq\G$. Such a setting---interesting especially with regard to computational concerns---is termed \emph{improper learning}. For example, prediction with side information (or, the \emph{supervised learning} problem) is one such case, where we choose $\Y \subset \reals$, $\Z = \X \times \Y$, $\F \subseteq \Y^\X = \G$ and $\ell(f,(x,y)) = |f(x)-y|$. \label{page:supervised} This setting will be studied later in the paper. Note that in the proper learning scenario, $\Val_T(\F,\Z) \geq 0$ (e.g. since all $z_t$'s can be chosen to be the same), and thus the ``$\limsup$" in Definition~\ref{def:learnability} can be simply replaced with the limit being equal to zero.

This paper is aimed at understanding the value of the game $\Val_T(\F, \Z)$ for various function classes $\F$. Since our focus is on the complexity of $\F$, we shall often write $\Val_T(\F)$ keeping the dependence on $\Z$ (and $\ell$) implicit. As we show, the sequential complexity notions---that were shown by \citet{RakSriTew13} to characterize uniform martingale Laws of Large Numbers---also give us a handle on the value $\Val_T(\F)$. In the next section, we briefly define these sequential complexity notions and mention some of the key relations between them. A more detailed account of the relationships between sequential complexity measures along with complete proofs can be found in \citep{RakSriTew13}.

%% file: seqcomplex.tex
\section{Sequential Complexities}

Unlike the well-studied statistical learning scenario with i.i.d.\ data, the online learning problem possesses a certain sequential dependence. Such dependence cannot be captured by classical notions of complexity that are based on a batch of data given as a \emph{tuple} of $T$ examples. A basic unit that does capture temporal dependence is a binary tree. Surprisingly, for the sequential prediction problems considered in this paper, one need not look further than binary trees to capture the relevant complexity.

A \emph{${\mathcal Z}$-valued tree $\z$ of depth $T$} is a complete rooted binary tree with nodes labeled by elements of $\Z$. Such a tree $\z$ is identified with the sequence $(\z_1,\ldots,\z_T)$ of labeling functions  $\z_i : \{\pm 1\}^{i-1} \to \mathcal{Z}$ which provide the labels for each node. Therefore, $\z_1\in\Z$ is the label for the \emph{root} of the tree, while $\z_i$ for $i>1$ is the label of the node obtained by following the path of length $i-1$ from the root, with $+1$ indicating `right' and $-1$ indicating `left'. A \emph{path} of length $T$ is given by the sequence $\epsilon = (\epsilon_1,\ldots,\epsilon_{T}) \in \{\pm1\}^{T}$. For brevity, we shall often write $\z_t(\epsilon)$, where $\epsilon=(\epsilon_1,\ldots,\epsilon_T)$, but it is understood that $\z_t$ depends only on the prefix $(\epsilon_1,\ldots,\epsilon_{t-1})$. 

Now, let  $\epsilon_1,\ldots,\epsilon_T$ be independent Rademacher random variables. Given a $\Z$-valued tree $\z$ of depth $T$, we define the \emph{sequential Rademacher complexity} of a function class $\G \subseteq \reals^\Z$ on a $\Z$-valued tree $\z$ as
$$
\Rad_T(\G,\z) ~\deq~ \En \left[ \sup_{g \in \G} \frac{1}{T}\sum_{t=1}^T \epsilon_t g(\z_t(\epsilon)) \right],
$$
and we denote by $\Rad_T(\G) = \sup_\z \Rad_T(\G,\z)$ its supremum over all $\Z$-valued trees of depth $T$. The importance of the introduced notion stems from the following result \citep[Theorem 2]{RakSriTew13}: for any distribution over a sequence  $(Z_1,\ldots,Z_T)$, we have
\begin{align}\label{eq:mainup}
	\En \left[ \sup_{g\in\G} \frac{1}{T}\sum_{t=1}^T \left( \En \left[g(Z_t)|Z^{t-1}\right] - g(Z_t) \right)  \right] \leq 2\, \Rad_T(\G) \ ,
\end{align}
where $Z^{t-1}=(Z_1,\ldots,Z_{t-1})$. In other words, the martingale version of the uniform deviations of means from expectations is controlled by the worst-case sequential Rademacher complexity. A matching lower bound also holds for the supremum over distributions on sequences in $\Z^T$. It then follows that a uniform martingale Law of Large Numbers holds for $\G$ if and only if $\Rad_T(\G)\to 0$. For i.i.d.\ random variables, a similar statement can be made in terms of the classical Rademacher complexity, and so one might hope that many other complexity notions from empirical process theory have martingale (or we may say, sequential) analogues. Luckily, this is indeed the case (see \cite{RakSriTew13}). As we show in this paper, these generalizations of the classical notions also give a handle on (as well as necessary and sufficient conditions for) online learnability, thus painting a picture that completely parallels statistical learning theory. But before we present our main results, let us recall some key definitions and results from \citep{RakSriTew13}.

In providing further upper bounds on sequential Rademacher complexity, the following definitions of an ``effective size'' of a function class generalize the classical notions of a covering number. A set $V$ of $\reals$-valued trees of depth $T$ is \emph{a (sequential) $\alpha$-cover} (with respect to $\ell_p$ norm) of $\G \subseteq \reals^\Z$ on a tree $\z$ of depth $T$ if
$$
\forall g \in \G,\ \forall \epsilon \in \{\pm1\}^T,\ \exists \v \in V \  \mrm{s.t.}  ~~~~ \left( \frac{1}{T} \sum_{t=1}^T |\v_t(\epsilon) - g(\z_t(\epsilon))|^p \right)^{1/p} \le \alpha .
$$
The \emph{(sequential) covering number} of a function class $\G$ on a given tree $\z$ is defined as 
$$
\N_p(\alpha, \G, \z) ~\deq~ \min\left\{|V| :  V \ \trm{is an }\alpha\text{-cover w.r.t. }\ell_p\trm{ norm of }\G \trm{ on } \z\right\}.
$$
It is straightforward to check that $\N_p(\alpha, \G, \z)\leq \N_q(\alpha, \G, \z)$ whenever $1\leq p\leq q\leq \infty$.

Further define $\N_p(\alpha, \G , T) = \sup_\z \N_p(\alpha, \G, \z) $, the maximal $\ell_p$ covering number of $\G$ over depth $T$ trees. For a class $\G$ of binary-valued functions, we also define a so-called \emph{$0$-cover} (or, cover at scale $0$), denoted by $\N(0,\G,\z)$, as equal to any $\N_p(0, \G, \z)$. The definition of a $0$-cover can be seen as the correct analogue of the \emph{size of a projection} of $\G$ onto a tuple of points in the i.i.d.\ case. The size of this projection in the i.i.d.\ case was the starting point of the work of Vapnik and Chervonenkis.

When $\G \subseteq [-1,1]^\Z$ is a \emph{finite} class of bounded functions, one can show \citep[Lemma 1]{RakSriTew13} that
\begin{align}\label{eq:fin}
\Rad_T(\G,\z)
\le \sqrt{\frac{2 \log |\G|}{T}} ,
\end{align}
a bound that should (correctly) remind the reader of the Exponential Weights regret bound. With the definition of an $\alpha$-cover with respect to $\ell_1$ norm, one can easily extend \eqref{eq:fin} beyond the finite case. Immediately from the definition of $\ell_1$ covering number, it follows that for any $\G\subseteq [-1,1]^\Z$, for any $\alpha>0$, 
\begin{align}
\Rad_T(\G,\z)
\le \alpha + \sqrt{\frac{2\log \N_1(\alpha, \G, \z)}{T}} 
\end{align}
\citep[Eq. (9)]{RakSriTew13}.
A tighter control is obtained by integrating the covering numbers at different scales. To this end, consider the following analogue of the Dudley entropy integral bound. For $p \ge 1$, the \emph{integrated complexity} of a function class $\G \subseteq [-1,1]^\Z$ on a $\Z$-valued tree of depth $T$ is defined as
\begin{align}
	\label{eq:integrated_complexity_def}
\Dudley^p_T (\G, \z) ~\deq~ \inf_{\alpha\geq 0}\left\{4 \alpha + \frac{12}{\sqrt{T}}\int_{\alpha}^{1} \sqrt{\log \ \mathcal{N}_p(\delta, \G, \z) \ } d \delta \right\} 
\end{align}
and $\Dudley^p_T(\G) = \sup_{\z} \Dudley^p_T(\G,\z),$ with $\Dudley^2_T(\G,\z)$ denoted simply by $\Dudley_T(\G,\z)$. We have previously
shown~\citep[Theorem 3]{RakSriTew13} that, for any function class $\G\subseteq [-1,1]^\Z$ and any $\Z$-valued tree $\z$ of depth $T$,
\begin{align}\label{eq:dudley}
\Rad_T(\G,\z) \le \Dudley_T(\G,\z)  .
\end{align}

We next turn to the description of sequential combinatorial parameters. A $\Z$-valued tree $\z$ of depth $d$ is \emph{shattered} by a function class $\G \subseteq \{\pm 1\}^{\Z}$ if for all $\epsilon \in \{\pm1\}^{d}$, there exists $g \in \G$ such that $g(\z_t(\epsilon)) = \epsilon_t$ for all $t \in [d]$. The \emph{Littlestone dimension} $\ldim(\G, \Z)$ is the largest positive integer $d$ such that $\G$ shatters a $\Z$-valued tree of depth $d$ \citep{Lit88, BenPalSha09}.  The scale-sensitive version of Littlestone dimension is defined as follows. A $\Z$-valued tree $\z$ of depth $d$ is \emph{$\alpha$-shattered} by a function class $\G \subseteq \reals^\Z$ if there exists an $\reals$-valued tree $\s$ of depth $d$ such that 
$$
\forall \epsilon \in \{\pm1\}^d , \ \exists g \in \G \ \ \ \trm{s.t. } \forall t \in [d], \  \epsilon_t (g(\z_t(\epsilon)) - \s_t(\epsilon)) \ge \alpha/2 .
$$
The tree $\s$ will be called a \emph{witness to shattering}. The \emph{ (sequential) fat-shattering dimension} $\fat_\alpha(\G, \Z)$ at scale $\alpha$ is the largest $d$ such that $\G$ $\alpha$-shatters a $\Z$-valued tree of depth $d$. 

The notions introduced above can be viewed as sequential generalizations of the VC dimension and the fat-shattering dimension where tuples of points get replaced by complete binary trees. In fact, one recovers the classical notions if the tree $\z$ in the above definitions is restricted to have the same values within a level (hence, no temporal dependence). Crucially, the sequential combinatorial analogues provide control for the growth of sequential covering numbers, justifying the definitions.

First, let $\G \subseteq {\{0,\ldots, k\}}^\Z$ be a class of functions with $\fat_2(\G) = d$. Then, it can be shown that \citep[Theorem 4]{RakSriTew13}, for any $T\geq 1$,
$$ \N_\infty(1/2, \G , T) \leq \sum_{i=0}^d {T\choose i} k^i \leq \left(ekT \right)^d .$$
For the second result \citep[Corollary 1]{RakSriTew13}, suppose $\G$ is a class of $[-1,1]$-valued functions on $\Z$. Then, for any $\alpha >0$, and any $T\geq 1$,
\begin{equation}\label{eq:fromcoveringtofat}
\N_\infty(\alpha, \G, T) \leq \left(\frac{2e T}{\alpha}\right)^{\fat_{\alpha} (\G) } .
\end{equation}
Finally, we recall a bound on the size of the 0-cover in terms of the $\fat_1$ combinatorial parameter \citep[Theorem 5]{RakSriTew13}. For a class $\G \subseteq {\{0,\ldots, k\}}^\Z$ with $\fat_1(\G) = d$, we have
	\begin{align} 
		\label{eq:ldim_generalization}
		\N (0, \G , T) \leq \sum_{i=0}^d {T\choose i} k^i \leq \left(ekT \right)^d \ .
	\end{align}
	In particular, for $k=1$ (that is, binary classification) we have $\fat_1(\G) = \ldim(\G)$. The inequality  \eqref{eq:ldim_generalization} is therefore a sequential analogue of the celebrated Vapnik-Chervonenkis-Sauer-Shelah lemma.

\section{Structural Properties}
\label{sec:structural}

For the examples developed in this paper, it will be crucial to exploit a number of useful properties that $\Rad_T(\G)$ satisfies. These properties allow one to establish online learnability for complex function classes even if no explicit learning algorithms are available. 

We first state some properties that are easily proved but are nevertheless very useful.

\begin{lemma}\label{lem:rad_properties}
Let $\F, \G \subseteq \reals^\Z$ and let $\conv(\G)$ denote the convex hull of $\G$. Let $\z$ be any $\Z$-valued tree of depth $T$. Then the following properties hold.
\begin{enumerate}
\item
If $\F\subseteq \G$, then $\Rad_T(\F,\z) \leq \Rad_T(\G,\z)$.
\item
$\Rad_T (\conv(\G),\z) = \Rad_T(\G,\z)$
\item
$\Rad_T(c\G,\z) = |c|\Rad_T(\G,\z)$ for all $c\in\reals$.
\item
For any $h: \Z \to \reals$, $\Rad_T(\G+h,\z) =  \Rad_T(\G,\z)$ where $\G+h = \{g+h: g\in\G\}$.
\end{enumerate}
\end{lemma}
These properties match those of the classical Rademacher complexity \citep{BarMed03} and can be proved in essentially the same way (we therefore skip the straightforward proofs). 

The next property is a key tool for many of the applications: it allows us to bound the sequential Rademacher complexity for the Cartesian product of function classes composed with a Lipschitz mapping in terms of complexities of the individual classes. 
\begin{lemma}\label{lem:inflip}
Let $\G = \G_1 \times \ldots \times \G_k$ where each $\G_j \subseteq [-1,1]^{\Z}$. Further, let $\phi : \reals^k \times \Z \to \reals$ be such that $\phi(\cdot,z)$ is $L$-Lipschitz with respect to $\|\cdot\|_\infty$ for all $z\in\Z$, and let
$$\phi\circ\G = \left\{z\mapsto \phi((g_1(z),\ldots,g_k(z)), z): g_j\in\G_j\right\}.$$ 
Then we have
$$
\textstyle\Rad_T(\phi \circ \G) \le 8\,L\,\left(1+ 4\sqrt{2}\log^{3/2}(eT^2)\right) \sum_{j=1}^k \Rad_T(\G_j) 
$$
as long as $\Rad_T(\G_j) \ge 1/T$ for each $j$. 
\end{lemma}

Let us explicitly state the more familiar contraction property, an immediate corollary of the above result.
\begin{corollary}
	\label{cor:contraction}
	Fix a class $\G\subseteq [-1,1]^\Z$ with $\Rad_T(\G)\ge 1/T$ and a function $\phi:\reals\times \Z\to\reals$.
	Assume $\phi(\cdot,z)$ is $L$-Lipschitz for all $z \in \Z$. Then
	$$ \Rad_T (\phi\circ\G) \leq 8\,L\,\left(1+4\sqrt{2}\log^{3/2}(eT^2)\right) \cdot\Rad_T(\G)$$
	where $\phi\circ\G = \{z \mapsto \phi(g(z),z): g\in \G\}$.
\end{corollary}

We state another useful corollary of Lemma~\ref{lem:inflip}.

\begin{corollary}
	\label{cor:radem_binary}
For a fixed binary function $b : \{\pm1\}^k \to \{\pm1\}$ and classes $\G_1, \ldots, \G_k$ of $\{\pm 1\}$-valued functions,
$$
\textstyle\Rad_T(b(\G_1,\ldots,\G_k)) \le \mc{O}\left(\log^{3/2}(T)\right) \sum_{j=1}^k \Rad_T(\G_j) 
$$
\end{corollary}

Note that, in the classical case, the Lipschitz contraction property holds without any extra poly-logarithmic factors in $T$ \citep{LedouxTalagrand91}. It is an open question whether the poly-logarithmic factors can be removed in the results above. It is worth pointing out ahead of time that Theorem~\ref{thm:valrad_supervised} below---in the setting of supervised learning with convex Lipschitz loss---does allow us to avoid the extraneous factor that would otherwise appear from a combination of Theorem~\ref{thm:valrad} and Corollary~\ref{cor:contraction}.

%% file: rademacher.tex
\section{Main Results}

We now relate the value of the game to the worst case expected value of the supremum of an empirical process. However, unlike empirical processes that involve i.i.d.\ sums, our process involves a sum of \emph{martingale differences}. In view of \eqref{eq:mainup}, the expected supremum can be further upper-bounded by the sequential Rademacher complexity.

\begin{theorem}\label{thm:valrad}
The minimax value is bounded as
$$
\frac{1}{T}\Val_T(\F) \le \sup_{\Prob}\En \sup_{g \in \loss(\F)} \left[ \frac{1}{T}\sum_{t=1}^T \left(  \vphantom{ \sum } 
\En [g(Z_t)|Z_{1},\ldots,Z_{t-1}]- g(Z_t)\right) \right] \le 2\, \Rad_T(\loss(\F))
$$
where $\loss(\F) = \{\loss(f,\cdot): f\in\F\}$ and the supremum is taken over all distributions $\Prob$ over $(Z_1,\ldots,Z_T)$.
\end{theorem}

We can now employ the tools developed earlier in the paper to upper bound the value of the game. Interestingly, any non-trivial upper bound guarantees \emph{existence} of a prediction strategy that has sublinear regret irrespective of the sequence of the moves of the adversary. This complexity-based approach of establishing learnability should be contrasted with the purely algorithm-based approaches found in the literature.

%% file: supervised.tex
\subsection{Supervised Learning}
\label{sec:supervised}

In this subsection we study the \emph{supervised learning problem} mentioned earlier in the paper. In this improper learning scenario, the learner at time $t$ picks a function $f_t:\X\to \reals$ and the adversary provides the input target pair $z_t=(x_t,y_t)\in \X\times \Y$ where $\Y\subset\reals$. In particular, the \emph{binary classification} problem corresponds to the case $\Y= \{\pm1\}$. Let $\F\subseteq \Y^\X$ and let us fix the absolute value loss function  $\loss(\hat{y},y) = |\hat{y} - y|$.  
While we focus on the absolute loss, it is easy to see that all the results hold (with modified rates) for any loss $\ell(\hat{y},y)$ such that for all $\hat{y}$ and $y$,  
$
\phi(\ell(\hat{y},y)) \le |\hat{y} - y| \le \Phi(\ell(\hat{y},y))
$
where $\Phi$ and $\phi$ are monotonically increasing functions. For instance, the squared loss $(\hat{y}-y)^2$ is a classic example. 

We now observe that the value of the improper supervised learning game can be equivalently written as
\begin{align}
	\label{eq:sup_value}
 \Val^{\trm{S}}_T(\F) = 
 \sup_{x_1} \inf_{q_1\in \tilde{\QD}} \sup_{y_1} \Ex_{\hat{y}_1 \sim q_1} \cdots~  \sup_{x_T} \inf_{q_T\in \tilde{\QD}}\sup_{y_T}  \Ex_{\hat{y}_T \sim q_T} \left[ \sum_{t=1}^T \loss(\hat{y}_t,y_t) - \inf_{f\in \F}\sum_{t=1}^T \loss(f(x_t),y_t) \right]
\end{align}
where $\tilde{\QD}$ denotes the set of probability distributions over $\Y$ and $\hat{y}_t$ has distribution $q_t$. This equivalence is easy to verify: we may view the choice $f_t:\X\to\Y$ as pre-specifying predictions $f_t(x)$ for all the possible $x\in\X$, while alternatively we can simply make the choice $\hat{y}_t\in\Y$ having observed the particular move $x_t\in\X$. The advantage of rewriting the game in the form \eqref{eq:sup_value} is that the minimax theorem only needs to be applied to the pair $\hat{y}_t$ and $y_t$, given the fixed choice $x_t$. The minimax theorem then holds even if weak compactness cannot be shown for the set of distributions on the original space of functions of the type $\X\to\Y$. 

An examination of the proof of Theorem~\ref{thm:valrad} reveals that the value \eqref{eq:sup_value} is upper bounded in exactly the same way, and the side information simply appears as an additional tree $\x$ in sequential Rademacher complexity, giving us:
\begin{align}
	\label{eq:sup_value_upper_bound_no_contraction}
	\frac{1}{T}\Val^{\trm{S}}_T(\F)\leq 2\sup_{\x,\y}\En \left[ \sup_{f\in\F}\frac{1}{T}\sum_{t=1}^T \epsilon_t \loss(f(\x_t(\epsilon)), \y_t(\epsilon))\right] \ .
\end{align}
However, for the supervised learning setting, we can strengthen Theorem~\ref{thm:valrad}. The following theorem allows us to remove any convex Lipschitz loss (including the absolute loss) before passing to the sequential Rademacher complexity.
\begin{theorem}\label{thm:valrad_supervised}
Let $\Y=[-1,1]$ and suppose, for any $y \in \Y$, $\loss(\cdot,y)$ is convex and $L$-Lipschitz. Then the minimax value of a supervised learning problem is upper bounded as
$$
\frac{1}{T} \Val^{\trm{S}}_T(\F) \le 2 \, L \, \Rad_T(\F) .
$$
\end{theorem}
We remark that the contraction property for sequential Rademacher complexity, stated in Section~\ref{sec:structural}, yields an extraneous logarithmic factor when applied to \eqref{eq:sup_value_upper_bound_no_contraction}; here, we achieve the desired bound by removing the Lipschitz function directly during the symmetrization step.

Armed with the theorem, we now prove the following result. 

\begin{proposition}\label{prop:uplow}
Consider the supervised learning problem with a function class $\F \subseteq [-1,1]^\X$ and absolute loss $\ell(\hat{y},y) = |\hat{y}-y|$. Then, for any $T\geq 1$, we have
\begin{align}\label{eq:uplow}
\frac{1}{4\sqrt{2}} \sup_{\alpha}\left\{\alpha \sqrt{\frac{\min\left\{\fat_{\alpha}, T\right\}}{T}} \right\} &\le \Rad_T(\F) \le \frac{1}{T}\Val^{\trm{S}}_T(\F) \le 2\Rad_T(\F) \leq 2\Dudley_T(\F) \notag\\
&~~~~\le~2\inf_{\alpha}\left\{4\alpha + \frac{12}{\sqrt{T}} \int_{\alpha}^{1} \sqrt{ \fat_\beta \log\left(\frac{2 e T}{\beta}\right)}\ d \beta \right\} \ ,
\end{align}
where $\fat_\alpha = \fat_\alpha(\F)$.
\end{proposition}

The proposition above implies that finiteness of the fat-shattering dimension at all scales is \emph{necessary and sufficient} for online learnability of the supervised learning problem. Further, all the complexity notions introduced so far are within a poly-logarithmic factor from each other whenever the problem is learnable. These results are summarized in the next theorem:
\begin{theorem}\label{thm:tight}
	For any function class $\F \subseteq [-1,1]^\X$, the following statements are equivalent
	\begin{enumerate}
		\item Function class $\F$ is online learnable in the supervised setting with absolute loss.
		\item Sequential Rademacher complexity satisfies $\lim_{T\to\infty} \Rad_T(\F) = 0$.
		\item For any $\alpha > 0$, the scale-sensitive dimension $\fat_\alpha(\F)$ is finite.
	\end{enumerate}
Moreover, if the function class is online learnable, then the value of the supervised game $\Val^{\trm{S}}_T(\F)$, the sequential Rademacher complexity $\Rad_T(\F)$, and the integrated complexity $\Dudley_T(\F)$ are within a multiplicative factor of $\mc{O}(\log^{3/2} T)$ of each other.
\end{theorem}

\begin{remark}
	Additionally, the three statements of Theorem~\ref{thm:tight} are equivalent to $\F$ satisfying a martingale version of the uniform Law of Large Numbers. This property is termed \emph{Sequential Uniform Convergence} by~\citet{RakSriTew13}, and we refer to their paper for more details.
\end{remark}

For binary classification, we write $\Val^{\trm{Binary}}_T$ for $\Val^{\trm{S}}_T$. This case has been investigated thoroughly by \citet{BenPalSha09} and indeed served as a key motivation for this paper. As a consequence of Proposition~\ref{prop:uplow} and Eq.~\eqref{eq:ldim_generalization}, we have a tight control on the value of the game for the binary classification problem. Note that the absolute loss in the binary
classification setting is simply the $0$-$1$ loss $\ell(\hat{y},y) = \ind{\hat{y} \neq y}$, where $\ind{\mathcal{U}}$ is $1$ if $\mathcal{U}$ is true and $0$ otherwise.
\begin{corollary}
For the binary classification problem with function class $\F$ and the $0$-$1$ loss, we have
$$
K_1  \sqrt{T \min\left\{\ldim(\F), T\right\}}  \le \Val^{\mathrm{Binary}}_T(\F)  \le K_2  \sqrt{T\ \ldim(\F) \log T}
$$
for some universal constants $K_1,K_2>0$.
\end{corollary}
Both the upper and the lower bound in the above result were originally derived in \cite{BenPalSha09}. Notably, we achieved the same bounds non-constructively through purely combinatorial and covering number arguments.

It is natural to ask whether being able to learn in the online model is different from learning in the i.i.d.\ model (in the distribution-free supervised setting). The standard example that exhibits a gap between the two frameworks (see, e.g., \cite{Lit88,BenPalSha09}) is binary classification using the class of step functions
$$\F = \left\{f_\theta(x)=\ind{x\leq \theta}: \theta\in[0,1]\right\}$$
on $[0,1]$. This class has VC dimension $1$, but is \emph{not} learnable in the online setting. Indeed, it is possible to verify that the Littlestone dimension is infinite. Interestingly, the closely-related class of ``ramp'' functions with slope $L>0$
	$$\F_L = \big\{f_\theta(x)= \ind{x \le \theta} + (1- L(x-\theta))\ind{\theta < x \le \theta + 1/L} : \theta\in[0,1]\big\}$$ 
\emph{is} learnable (say for supervised learning using absolute loss) in the online setting (and hence also in the i.i.d.\ case). Furthermore, the larger class of all bounded $L$-\emph{Lipschitz} functions on a bounded interval is also online learnable (see Eq.~\eqref{eq:bound_by_metric_entropy} and proof of Proposition~\ref{prop:isotron}). Once again, we are able to make these statements from purely complexity-based considerations, without exhibiting an algorithm. Further examples where we can demonstrate online learnability are explored in Section~\ref{sec:examples}.

\subsection{Online Convex Optimization}

Over the past decade, Online Convex Optimization (OCO) has emerged as a unified online learning framework \citep{Zinkevich03,shalev2011online}. Various methods, such as Exponential Weights, can be viewed as instances of online mirror descent, solving the associated OCO problem. Much research effort has been devoted to understanding this abstract and simplified setting. It is tempting to say that any problem of online learning, as defined in the Introduction, can be viewed as OCO (in fact, online \emph{linear} optimization) over the set of probability distributions; however, one should also recognize that by linearizing the problem, any interesting structure is lost and one instead suffers from the possibly unnecessary dependence on the number of functions in the class $\F$. Nevertheless, OCO is a central part of the recent literature, and we will study this scenario using techniques developed in this paper.

The standard setting of online convex optimization is as follows. The set of moves of the learner $\F$ is a bounded closed convex subset of a Banach space $(\mc{B}, \|\cdot\|)$ with $\|f\| \le D$ for all $f \in \F$ (the reader can think of $\reals^d$ equipped with an $\ell_p$ norm for simplicity). Let $\|\cdot\|_\star$ be the dual norm. The adversary's set $\Z$ consists of convex $G$-Lipschitz (with respect to $\|\cdot\|_\star$) functions over $\F$:
\[
	\Z = \Zcvx = \left\{ g:\F \to \reals \::\: g \text{ convex and } G\text{-Lipschitz w.r.t. } \|\cdot\|_\star \right\} \ .
\]
Let the loss function be $\loss(f,g) = g(f)$, the evaluation of the adversarially chosen function at $f$. For the particular case of online \emph{linear} optimization, we instead take
\[
	\Z = \Zlin = \{ f \mapsto \inner{f, z} \::\: \|z\|_\star \le G \} 
\]
with $\Z$ now a subset of the dual space. It is well-known (see, e.g., \cite{abernethy08optimal}) that the online convex optimization problem (without further assumptions on the functions in $\Zcvx$) is as hard as the corresponding linear optimization problem with $\Zlin$ if one considers deterministic algorithms. The same trivially extends to randomized methods:
\begin{lemma}
	\label{lem:equal_value}
Suppose $\F, \Zcvx, \Zlin$ be defined as above. Then we have
\[
	\Val_T(\F,\Zcvx) = \Val_T(\F,\Zlin) \ .
\]
\end{lemma}

We will now show how to use the above result to derive minimax regret guarantees for OCO. The reader may wonder why we do not directly try to bound the value $\Val_T(\F,\Zcvx)$ by $\Rad_T(\F,\Zcvx)$. In fact, this proof strategy cannot give a non-trivial bound if $\F$ is a subset of a high-dimensional (or infinite-dimensional) space \citep[Sec. 4.1]{ShalevShSrSr09}. Instead, we use the lemma above to bound the value of the game where adversary plays convex functions with that of the game where adversary plays linear functions.

A function $\Psi: \F \to \reals$
is $(\sparam,q)$-uniformly convex (for $q\in[2,\infty)$) on $\F$ with respect to a norm $\|\cdot\|$ if, for all $\theta \in [0,1]$ and $f_1,f_2 \in \F$,
\[
	\Psi( \theta f_1 + (1-\theta) f_2 ) \le \theta \Psi(f_1) + (1-\theta) \Psi(f_2) - \frac{\sparam\,\theta\,(1-\theta)}{q} \| f_1 - f_2 \|^q \ .
\]
A $(\sparam,2)$-uniformly convex function will be called $\sparam$-strongly convex. 

We will give examples shortly but we first state a proposition that is useful to bound the sequential Rademacher complexity of  linear function classes. The crucial duality fact exploited in its proof is that $\Psi$ is $(\sparam,q)$-uniformly convex
with respect to $\|\cdot\|$ if and only if $\Psi^\star$ is $(1/\sparam,p)$-uniformly smooth with respect to $\|\cdot\|_\star$ where $1/p+1/q =1$.

\begin{proposition}[\cite{RakSriTew13}]
	\label{prop:rad_linear_functions}
Let $\F$ be a subset of some Banach space $\mc{B}$ with norm $\|\cdot\|$ and let $\Z$ be a subset of the dual space $\mc{B}^\star$ equipped with norm $\|\cdot\|_\star$. Suppose that $\Psi:\F \to \reals$ is $(\sparam,q)$-uniformly convex with respect to $\|\cdot\|$ and $0 \le \Psi(f) \le \Psimax$ for all $f \in \F$. Then we have
\[
	\Rad_T(\F) \le C_p \|\Z\|_\star \left( \frac{\Psimax^{p-1} }{\sparam \, T^{p-1}} \right)^{1/p} ,
\]
where $\|\Z\|_\star = \sup_{z \in \Z}\ \|z\|_\star$, $p$ is such that $1/p+1/q=1$, and $C_p = (p/(p-1))^{\frac{p-1}{p}}$.
\end{proposition}

Using the above Proposition in conjunction with Lemma~\ref{lem:equal_value} and Theorem~\ref{thm:valrad}, we can immediately conclude that
$$
\Val_T(\F,\Zcvx) \le 2\,T\,\Rad_T(\F) \le 2G\sqrt{ \frac{2\,\Psimax\,T}{\sigma}}$$
for any non-negative function $\Psi:\F \to \reals$ that is $\sparam$-strongly convex w.r.t.
$\|\cdot\|$. Note that, typically, $\Psimax$ will depend on $D$. For example, in the particular case when $\|\cdot\| = \|\cdot\|_\star = \|\cdot\|_2$,
we can take $\Psi(u) = \tfrac{1}{2}\|u\|_2^2$ and the above bound becomes $2GD\sqrt{T}$ and recovers the guarantee for
the online gradient descent algorithm. In general, for $\|\cdot\| = \|\cdot\|_p$ and $\|\cdot\|_\star = \|\cdot\|_q$, we can use
$\Psi(u) = \tfrac{1}{2}\|u\|_p^2$ to get a bound of $2GD\sqrt{T/(p-1)}$ since $\Psi$ is $(p-1)$-strongly convex w.r.t. $\|\cdot\|_p$.
These $\mathcal{O}(\sqrt{T})$ regret rates are not new but we re-derive them to illustrate the usefulness of the tools we developed.

%% file: examples.tex
\section{Further Examples}
\label{sec:examples}

Now we present some further applications of the tools we have developed in this paper for some specific learning
problems.  To begin, we show how to bound the sequential Rademacher complexity of functions computed by neural networks. Then, we derive margin based regret bounds in a fairly general setting.  The classical analogues of these margin bounds have played a big role in the modern theory of
supervised learning where they help explain the success of linear classifiers in high dimensional spaces (see, for
example,  \cite{SchapireFrBaLe97,KoltchinskiiPa02}).  We then study the complexity of classes formed by decision trees, analyze the setting of transductive learning, and consider an online version of the Isotron problem. Finally, we make a connection to the seminal work of
\cite{CesaBianLugo99} by re-deriving their bound on the minimax regret in a static experts game in terms of the classical 
Rademacher averages.

\subsection{Neural Networks}
We provide below a bound on the sequential Rademacher complexity for  classic multi-layer neural networks thus showing they are learnable in the online setting. The model of neural networks we consider below and the bounds we provide are analogous to the ones considered in the i.i.d.\ setting by \citet{BarMed03}. 

Consider a $k$-layer $1$-norm neural network, defined by a base function class $\F_1$ and, recursively, for each $2 \le i \le k$,
$$
\F_i = \left\{x \mapsto \sum_{j} w^{i}_j \sigma\left( f_j(x)\right) ~~\Big|~~ \forall j\   f_j \in \F_{i-1} , \|w^{i}\|_1 \le B_i\right\} \ ,
$$
where $\sigma$ is a Lipschitz transfer function, such as the sigmoid function.

\begin{proposition}
	\label{prop:NN}
Suppose $\sigma : \reals \to [-1,1]$ is $L$-Lipschitz with $\sigma(0)=0$. Then it holds that
$$
\Rad_T(\F_k) \le \left(\prod_{i=2}^k 16 B_i  \right) L^{k-1} \left(1+4\sqrt{2}\log^{3/2}(eT^2)\right)^{k}  \Rad_T(\F_1).
$$
In particular, for the case of 
$$
\textstyle\F_1 = \left\{x \mapsto \sum_{j} w^{1}_j x_j  ~~\Big|~~ \|w\|_1 \le B_1\right\}
$$
and $\X \subset \reals^d$ we have the bound
$$
\Rad_T(\F_k) \le \left(\prod_{i=1}^k 16 B_i \right) L^{k-1} \left(1+4\sqrt{2}\log^{3/2}(eT^2)\right)^{k} X_\infty  \sqrt{\frac{2  \log d}{T}}
$$
where $X_\infty$ is such that $\forall x \in \X$, $\|x\|_\infty \le X_\infty$.
\end{proposition}
Our result is a non-constructive guarantee, and, to the best of our knowledge, no algorithms for learning neural networks within the online learning model exist. It is not clear if the above bounds could be obtained via computationally efficient methods.

\subsection{Margin Based Regret}
In the classical statistical setting, margin bounds provide guarantees on the expected zero-one loss of a classifier based on the empirical margin zero-one error. These results form the basis of the theory of large margin classifiers (see \cite{SchapireFrBaLe97,KoltchinskiiPa02}). Recently, in the online setting, bounds of a similar flavor have been shown through the concept of margin via the Littlestone dimension \citep{BenPalSha09}. We show that our machinery can easily lead to margin bounds for binary classification problems for general function classes $\F$ based on their sequential Rademacher complexity. We use ideas from \citep{KoltchinskiiPa02} to do this.

\begin{proposition}
	\label{prop:margin}
For any function class $\F \subset [-1,1]^{\X}$, there exists a randomized prediction strategy given by $\tau$ such that for any sequence $z_1,\ldots,z_T$ where each $z_t = (x_t,y_t) \in \X \times \{\pm 1\}$, 
\begin{align*}
&\sum_{t=1}^T \En_{\hat{y}_t \sim \tau_t(z_{1:t-1})} \left[ \ind{\hat{y}_t y_t < 0 } \right]   \\
&\le \inf_{\gamma > 0}\left\{\inf_{f \in \F} \sum_{t=1}^T \ind{f(x_t) y_t < 2\gamma} +  \frac{16}{\gamma} \left(1+4\sqrt{2}\log^{3/2}(eT^2)\right) T\Rad_T(\F) + 2\sqrt{T}\left(1 + \log \log\left(\frac{1}{\gamma}\right) \right)\right\}
\end{align*}
\end{proposition}
To interpret the above bound, suppose that the sequence of $y_t$'s is predicted with a margin $2\gamma$ by some function $f\in\F$. The upper bound guarantees that there exists a strategy (that does not need to know the value of $\gamma$) with cumulative loss given by the sequential Rademacher complexity of $\F$ divided by the margin, up to poly-logarithmic factors. Crucially, the bound does not directly depend on the dimensionality of the input space $\X$.

\subsection{Decision Trees}
We consider here the binary classification problem where the learner competes with a set of decision trees of depth no more than $d$. The function class $\F$ for this problem is defined as follows. Each $f\in\F$ is defined by choosing a rooted binary tree of depth no more than $d$ and associating to each node a binary valued decision function from a set $\mc{H} \subseteq \{\pm 1\}^\X$. A binary value for a given $x$ can be obtained by traversing the tree from the root according to the value of the decision function at each node and then reading off the label of the leaf. Importantly, $x$ ``reaches'' only one leaf of the tree. Alternatively, for any  leaf $l$, the membership of $x$ is given by the conjunction 
 $$
 \prod_i \ind{h_{l,i}(x)=1}
 $$
where $h_{l,i}$ is either the decision function at node $i$ along the path to the leaf $l$, or its negation. To complete the definition of $f$, we choose weights $w_l > 0$,  $\sum_l w_l = 1$, along with the value $\sigma_l \in \{\pm 1\}$ of the function on each leaf $l$. The resulting function $f$ can be written as
$$f(x) = \sum_{l} w_l \sigma_l \prod_i \ind{h_{l,i}(x)=1}$$
where the sum runs over all the leaves of the tree.

The following proposition is the online analogue of a result about decision tree learning that \citet{BarMed03} proved in the i.i.d.\ setting.

\begin{proposition}
	\label{prop:DT}
Denote by $\F$ the class of decision trees of depth at most 
$d$ with decision functions in $\mc{H}$. There exists a randomized strategy $\tau$ for the learner such that for any sequence of instances $z_1, \ldots, z_T$, with  $z_t = (x_t,y_t) \in \X \times \{\pm 1\}$,
\begin{align*}
&\sum_{t=1}^T \Es{\hat{y}_t \sim \tau_t(z_{1:t-1})}{\ind{\hat{y}_t \ne y_t }} \le \inf_{f \in \F} \sum_{t=1}^T  \ind{f(x_t) \ne y_t} \\
& ~~~~~~~~~~~~~~~~~ + \mc{O}\left(\sum_{l} \min \left(C(l) , d \log^{3}(T)\ T\ \Rad(\mc{H})\right) + \sqrt{T}\log(N)\right) 
\end{align*}
where $C(l)$ denotes the number of instances that reach the leaf $l$ and are correctly classified in the decision tree $f$ that minimizes  $\sum_{t=1}^T  \ind{y_t f(x_t) \leq 0}$, with $N>2$ being the number of leaves in this tree. 
\end{proposition}

It is not clear whether computationally feasible online methods exist for learning decision trees, and this represents an interesting avenue of further research.

\subsection{Transductive Learning}
\label{sec:transductive}
Let $\F$ be a class of functions from $\X$ to $\reals$. Let 
\begin{align}
	\label{eq:def_metric_entropy}
\Nhat_\infty (\alpha, \F) = \min\left\{|G| : G\subseteq \reals^\X \trm{ s.t. }  \forall f\in\F \ \ \exists g\in G ~\trm{ satisfying }~ \|f-g\|_\infty \leq \alpha \right\}
\end{align}
be the $\ell_\infty$ covering number at scale $\alpha$, where the cover is pointwise on all of $\X$. It is easy to verify that  
\begin{align}
	\label{eq:bound_by_metric_entropy}
	\forall T, \ \ \ \N_\infty(\alpha, \F, T) \leq \Nhat_\infty(\alpha, \F) \ .
\end{align}
Indeed, let $G$ be a minimal cover of $\F$ at scale $\alpha$. We claim that for any $\X$-valued tree of depth $T$, the set $V=\{\v^g = g \circ \x: g\in G\}$ of $\reals$-valued trees is an $\ell_\infty$ cover of $\F$ on $\x$. Fix any $\epsilon\in\{\pm1\}^T$ and $f\in\F$, and let $g\in G$ be such that $\|f-g\|_\infty \leq \alpha$. Clearly $|\v^g_t(\epsilon)-f(\x_t(\epsilon))|\leq \alpha$ for any $1\leq t\leq T$, concluding the proof.

This simple observation can be applied in several situations. First, consider the problem of {\em transductive learning}, where the set $\X=\{x_1,\ldots,x_n\}$ is a finite set. To ensure online learnability, it is sufficient to consider an assumption on the dependence of $\Nhat_\infty (\alpha, \F)$ on $\alpha$. An obvious example of such a class is a VC-type class with $\Nhat_\infty(\alpha, \F)  \leq (c/\alpha)^d$ for some $c$ which can depend on $n$. Assume that $\F\subset [-1,1]^\X$. Substituting this bound on the covering number into Eq.~\eqref{eq:integrated_complexity_def} and choosing $\alpha =0$, we observe that the value of the supervised game is upper bounded by $2\Dudley_T(\F) \leq 48\,\sqrt{d T\log c}$ by Proposition~\ref{prop:uplow}. It is easy to see that if $n$ is fixed and the problem is learnable in the batch (i.e. i.i.d.) setting, then the problem is learnable in the online transductive model.

In the transductive setting considered by \citet{KakadeKalai05}, it is assumed that $n\leq T$ and $\F$ consists of binary-valued functions. If $\F$ is a class with VC dimension $d$, the Sauer-Shelah lemma ensures that the $\ell_\infty$ cover is smaller than $(en/d)^d \leq (eT/d)^d$. Using the previous argument with $c=eT$, we obtain a bound of $4\sqrt{dT\log (eT)}$ for the value of the game, matching the bound of \cite{KakadeKalai05} up to a constant factor.

\subsection{Isotron}
\label{sec:isotron}

\citet{KalSas09} introduced a method called \emph{Isotron} for learning Single Index Models (SIM). These models generalize linear and logistic regression, generalized linear models, and classification by linear threshold functions. For brevity, we only describe the Idealized SIM problem considered by the authors. In its ``batch'' version, we assume that the data are revealed at once as a set  $\{(x_t, y_t)\}_{t=1}^T \in \reals^d \times \reals$ where $y_t= u(\inner{w,x_t})$ for some unknown $w\in \reals^d$ of bounded norm and an unknown non-decreasing $u:\reals\to\reals$ with a bounded Lipschitz constant. Given this data, the goal is to iteratively find the function $u$ and the direction $w$, making as few mistakes as possible. The error is measured as $\frac{1}{T}\sum_{t=1}^T (f_i(x_t) - y_t)^2$, where $f_i(x) = u_i(\inner{w_i,x})$ is the iterative approximation found by the algorithm on the $i$th round. The elegant computationally efficient method presented by \citet{KalSas09} is motivated by Perceptron, and a natural open question posed by the authors is whether there is an online variant of Isotron. Before even attempting a quest for such an algorithm, we can ask a more basic question: is the (Idealized) SIM problem even learnable in the online framework? After all, most online methods deal with convex functions, but $u$ is only assumed to be Lipschitz and non-decreasing. We answer the question easily with the tools we have developed.

We are interested in online learnability of
\begin{align} 
	\label{eq:def_isotron_class}
	\cH = \left\{ f(x,y) = (y-u(\inner{w, x}))^2 \ | \ u:[-1,1]\to [-1,1] \mbox{ $1$-Lipschitz }, \ \|w\|_2\leq 1 \right\}
\end{align}
in the supervised setting, over $\X = B_2$ (the unit Euclidean ball in $\reals^d$) and $\Y=[-1,1]$. In particular, we prove the result for Lipschitz, but not necessarily non-decreasing functions. It is evident that $\cH$ is a composition with three levels: the squared loss, the Lipschitz non-decreasing function, and the linear function. The proof of the following proposition shows that the covering number of the class does not increase much under these compositions.

\begin{proposition}\label{prop:isotron}
	The class $\cH$ defined in \eqref{eq:def_isotron_class} is online learnable in the (improper) supervised learning setting. Moreover, the minimax regret is $$\mc{O}(\sqrt{T}\log^{3/2} (T)).$$ 
\end{proposition}

Once again, it is not clear whether a computationally efficient method attaining the above guarantee exists.

\subsection{Prediction of Individual Sequences with Static Experts}
\label{sec:staticexperts}

We also consider the problem of prediction of individual sequences, which has been studied both in information theory and in learning theory. In particular, in the case of binary prediction, \citet{CesaBianLugo99} proved upper bounds on the minimax value in terms of the (classical) Rademacher complexity and the (classical) Dudley integral. One of the assumptions made by \citet{CesaBianLugo99} is that experts are \emph{static}. That is, their prediction only depends on the current round, not on the past information. Formally, we define static experts as vectors $\bar{f} = (f_1,\ldots,f_T) \in [0,1]^T$, and let $\F$ denote a class of such experts. Let $\Y = \{0,1\}$, putting us in the scenario of binary classification with no side information. Then regret on a particular sequence $y_1,\ldots,y_T$ can be written as
\[
	\sum_{t=1}^T \ell_t(\bar{f}_t,y_t) - \inf_{\bar{f} \in \F} \sum_{t=1} \ell_t(\bar{f},y_t)
\]
where $\bar{f}_t$ is the expert chosen by the learning algorithm at time $t$. Observe that the proof of Theorem~\ref{thm:valrad}
does not require the loss to be time independent. In the case of absolute loss, the Rademacher complexity appearing
on the right hand side in Theorem~\ref{thm:valrad} becomes
\[
	\sup_{\y} \Es{\epsilon}{\sup_{\bar{f} \in \F} \sum_{t=1}^T \epsilon_t \ell_t(\bar{f},\y_t(\epsilon)) } =
	\sup_{\y} \Es{\epsilon}{\sup_{\bar{f} \in \F} \sum_{t=1}^T \epsilon_t |f_t-\y_t(\epsilon)| } \ .
\]
where the supremum is over all $\Y$-valued trees of depth $T$. Noting that for $f \in [0,1], y \in \{0,1\}$, $|f - y|$ can be written as $(1-2y)f + y$, the above equals
\begin{align*}
	\sup_{\y} \Es{\epsilon}{\left(\sup_{\bar{f} \in \F} \sum_{t=1}^T \epsilon_t (1-2\y_t(\epsilon)) f_t \right)
	+ \sum_{t=1}^T \epsilon_t \y_t(\epsilon) } 
=	
	\sup_{\y} \Es{\epsilon}{\sup_{\bar{f} \in \F} \sum_{t=1}^T \epsilon_t (1-2\y_t(\epsilon)) f_t }
\end{align*}
It can be easily verified that the joint distribution of $\{ \epsilon_t(1-2\y_t(\epsilon)) \}_{t=1}^T$ is still
i.i.d. Rademacher and hence the value of the game is upper bounded by
\[
	2 \Es{\epsilon}{ \sup_{\bar{f} \in \F} \sum_{t=1}^T \epsilon_t f_t } \ ,
\]
recovering the upper bound of Theorem 3 in \citep{CesaBianLugo99}. We note that for this particular scenario, the factor of $2$ (that appears because of symmetrization) is not needed. This factor is the price we pay for deducing the result from the general statement of Theorem~\ref{thm:valrad}.

%% file: discussion.tex
\section{Discussion}

The tools provided in this paper allow us to establish existence of regret minimization algorithms by working directly with the minimax value. The non-constructive nature of our results is due to the application of the minimax theorem: the dual strategy does not give a handle on the primal strategy. Furthermore, by passing to upper bounds on the dual formulation~\eqref{eq:value_equality} of the value of the game, we remove the dependence on the dual strategy altogether. After the original paper \citep{RakSriTew10a} appeared, the algorithmic approach has been developed by \cite{Rakhlin2012Relax} who showed that the prediction for round $t$ can be obtained by appealing to the minimax theorem for rounds $t+1$ to $T$, yet keeping the minimax expression for round $t$ as is. The notion of a relaxation (in the spirit of approximate dynamic programming) then allowed the authors to develop a general recipe for deriving computationally feasible prediction methods. The techniques of the present paper form the basis for the algorithmic developments in \citep{Rakhlin2012Relax}. We refer the reader to \citep{StatNotes2012,Rakhlin2012Relax} for details.

%% file: minimax.tex
\section{A Minimax Theorem}
\label{sec:minimax_appendix}

The minimax theorem is one of this paper's main workhorses. For completeness, we state a general version of this theorem --- the von Neumann-Fan minimax theorem --- due to \cite{borwein2014very} (see also \citep{borwein1986fan}).

\begin{theorem}[\cite{borwein2014very}] Let $\mathcal{A}$ and $\mathcal{B}$ be Banach spaces. Let $A\subset \mathcal{A}$ be nonempty, weakly compact, and convex, and let $B\subset \mathcal{B}$ be nonempty and convex. Let $g:\mathcal{A}\times\mathcal{B}\to \reals$ be concave with respect to $b\in B$ and convex and lower-semicontinuous with respect to $a\in A$, and weakly continuous in $a$ when restricted to $A$. Then
	\begin{align}
		\label{eq:minimax_borwein}
		\sup_{b\in B} \inf_{a\in A} g(a,b) = \inf_{a\in A} \sup_{b\in B} g(a,b).
	\end{align}
\end{theorem}

In the proof of Theorem~\ref{thm:minimax}, the minimax theorem is invoked to assure that
\begin{align}
	\label{eq:minimax_application}
	\inf_{q_t\in \QD} \sup_{p_t\in\PD} \En \left[ \loss(f_t,z_t) + \xi(z_t)\right] = \sup_{p_t\in\PD} \inf_{q_t\in \QD}  \En \left[ \loss(f_t,z_t) + \xi(z_t)\right]
\end{align}
where $\xi(z_t)$ is a rather complicated function that includes the repeated infima and suprema from steps $t+1$ to $T$ of regret expression that includes the variable $z_t$ (but not $f_t$). The expectation in \eqref{eq:minimax_application} is with respect to $f_t\sim q_t$ and $z_t\sim p_t$. To apply \eqref{eq:minimax_borwein}, we take $g$ to be the bilinear form in $q_t$ and $p_t$, with $A=\QD$ and $B=\PD$. Equipped with the total variation distance, $\QD$ and $\PD$ can be seen as subsets of a Banach space of measures on $\F$ and $\Z$, respectively. In terms of conditions, it is enough to check weak compactness of $\QD$ and assume continuity of the loss function (lower semi-continuity can be used as well). 

Weak compactness of the set of probability measures on a complete separable metric space is equivalent to uniform tightness by the fundamental result of Prohorov (see e.g. \cite[Theorem 8.6.2.]{bogachev2007measure}, \citep{VaartWellner96}). If $\F$ itself is compact, then the set $\Delta(\F)$ of probability measures on $\F$ is tight, and hence (under the continuity of the loss) the minimax theorem holds. If $\F$ is not compact, tightness can be established under the following general condition. According to Example 8.6.5 (ii) in \cite{bogachev2007measure}, a family $\Delta(\F)$ of Borel probability measures on a separable \emph{reflexive} Banach space $E$ is uniformly tight (under the weak topology) precisely when there exists a function $V : E \to [0, \infty)$ continuous in the norm topology such that
$$ \lim_{\|f\|\to \infty} V(f) = \infty ~~~\mbox{ and }~~~ \sup_{q\in\Delta(\F)} \En_{f\sim q} V(f) < \infty .$$
As an example, if $\F$ is a subset of a ball in $E$, it is enough to take $V(f)=\|f\|$. 

Finally, we remark that in the supervised learning case by considering the improper learning scenario we allow $x_t$ to be observed before the choice $\hat{y}_t$ is made. Therefore, we do not need to invoke the minimax theorem on the space of functions $\F$, but rather (see the proof of Theorem~\ref{thm:valrad_supervised}) for two real-valued decisions in a bounded interval. This makes the application of the minimax theorem straightforward.

%% file: appendix.tex
\section{Proofs}

\begin{proof}[\textbf{of Theorem~\ref{thm:minimax}}]
	For brevity, denote $\psi(z_{1:T}) = \inf_{f\in\F} \sum_{t=1}^T \loss(f,z_t)$.
	The first step in the proof is to appeal to the minimax theorem for every couple of $\inf$ and $\sup$:
\begin{align*}
		\Val_T(\F)
		&= \inf_{q_1}\sup_{p_1} \En_{\underset{z_1 \sim p_1}{f_1\sim q_1}}\ldots\inf_{q_T}\sup_{p_T} \En_{\underset{z_T \sim p_T}{f_T\sim q_T}} \left\{ \sum_{t=1}^T \loss(f_t,z_t) - \psi(z_{1:T}) \right\} \\
	&= \sup_{p_1} \inf_{q_1} \En_{\underset{z_1 \sim p_1}{f_1\sim q_1}}\ldots\sup_{p_T} \inf_{q_T} \En_{\underset{z_T \sim p_T}{f_T\sim q_T}} \left\{ \sum_{t=1}^T \loss(f_t,z_t) - \psi(z_{1:T}) \right\} \\
	&=\sup_{p_1}\inf_{f_1}\En_{z_1\sim p_1} \ldots \sup_{p_T}\inf_{f_T} \En_{z_T\sim p_T} \left\{ \sum_{t=1}^T \loss(f_t,z_t) - \psi(z_{1:T})\right\} 
\end{align*}
where $q_t$ and $p_t$ range over $\QD$ and $\PD$, the sets of distributions on $\F$ and $\Z$, respectively. From now on, it will be understood that $z_t$ has distribution $p_t$. By moving the expectation with respect to $z_T$ and then the infimum with respect to $f_T$ inside the expression, we arrive at
\begin{align}
	&\sup_{p_1}\inf_{f_1}\Ex_{z_1} \ldots \sup_{p_{T-1}}\inf_{f_{T-1}}\Ex_{z_{T-1}}\sup_{p_T} \left\{ \sum_{t=1}^{T-1} \loss(f_t,z_t) + \left[\inf_{f_T}\Ex_{z_T} \loss(f_T,z_T) \right]- \Ex_{z_T}\psi(z_{1:T})\right\} \notag \\
	&=\sup_{p_1}\inf_{f_1}\Ex_{z_1} \ldots \sup_{p_{T-1}}\inf_{f_{T-1}}\Ex_{z_{T-1}}\sup_{p_T} \Ex_{z_T} \left\{ \sum_{t=1}^{T-1} \loss(f_t,z_t) + \left[\inf_{f_T}\Ex_{z_T} \loss(f_T,z_T) \right]- \psi(z_{1:T})\right\} \label{eq:pullingstep}
\end{align}
Let us now repeat the procedure for step $T-1$. The above expression is equal to
\begin{align*}
	&\sup_{p_1}\inf_{f_1}\Ex_{z_1} \ldots \sup_{p_{T-1}}\inf_{f_{T-1}}\Ex_{z_{T-1}} \left\{ \sum_{t=1}^{T-1} \loss(f_t,z_t) + \sup_{p_T}\Ex_{z_T} \left[ \inf_{f_T}\Ex_{z_T} \loss(f_T,z_T)- \psi(z_{1:T})\right]\right\} 
\end{align*}
which, in turn, is equal to
\begin{align*}
	&\sup_{p_1}\inf_{f_1}\Ex_{z_1} \ldots \sup_{p_{T-1}} \left\{ \sum_{t=1}^{T-2} \loss(f_t,z_t) + \left[\inf_{f_{T-1}} \Ex_{z_{T-1}} \loss(f_{T-1},z_{T-1}) \right] \right.\\
	&\left.\hspace{1.9in} + \Ex_{z_{T-1}} \sup_{p_T} \Ex_{z_T} \left[ \inf_{f_T}\Ex_{z_T} \loss(f_T,z_T)- \psi(z_{1:T})\right]\right\} \\
	&=\sup_{p_1}\inf_{f_1}\Ex_{z_1} \ldots \sup_{p_{T-1}}\Ex_{z_{T-1}} \sup_{p_T} \Ex_{z_T} \left\{ \sum_{t=1}^{T-2} \loss(f_t,z_t) + \left[\inf_{f_{T-1}} \Ex_{z_{T-1}} \loss(f_{T-1},z_{T-1})\right] \right.\\
	&\left.\hspace{2.9in} + \left[ \inf_{f_T}\Ex_{z_T} \loss(f_T,z_T) \right]- \psi(z_{1:T})\right\}
\end{align*}
Continuing in this fashion for $T-2$ and all the way down to $t=1$ proves the theorem.
\end{proof}

\begin{proof}[\textbf{of Lemma~\ref{lem:inflip}}]
Without loss of generality assume that the Lipschitz constant $L=1$, as the general case follows by scaling $\phi$. Fix a $\Z$-valued tree $\z$ of depth $T$. We first claim that 
$$\log \ \mathcal{N}_2(\beta, \phi \circ \G, \z )  \le \sum_{j=1}^k \log \ \mathcal{N}_\infty (\beta, \G_j, \z ) \ .$$ 
Suppose $V_1, \ldots, V_k$ are minimal $\beta$-covers with respect to $\ell_\infty$ for $\G_1, \ldots, \G_k$ on the tree $\z$. Consider the set
\[
	V^\phi = \{ \v^\phi \::\: \v \in V_1 \times \ldots \times V_k \}
\]
where $\v^\phi$ is the tree such that $\v^\phi_t(\epsilon) = \phi(\v_t(\epsilon),\z_t(\epsilon))$.
Then, for any $g = (g_1, \ldots,g_k) \in \G$ and any $\epsilon \in \{\pm1\}^T$, with representatives $(\v^1,\ldots,\v^k) \in V_1 \times \ldots \times V_k$, we have,
\begin{align*}
&\sqrt{\frac{1}{T} \sum_{t=1}^T \left( \phi(g(\z_t(\epsilon)),\z_t(\epsilon)) - \v^\phi_t(\epsilon) \right)^2} 
\leq \max_{t\in[T]} \left| \phi(g(\z_t(\epsilon)),\z_t(\epsilon)) - \v^\phi_t(\epsilon) \right| \\ 
&=\max_{t\in[T]} \left| \phi(g(\z_t(\epsilon)),\z_t(\epsilon)) - \phi(\v_t(\epsilon),\z_t(\epsilon)) \right| 
 \le  \max_{j \in [k]} \max_{t \in[T]}|g_j(\z_t(\epsilon))) - \v^j_t(\epsilon) | \le \beta
\end{align*}
Thus we see that $V^\phi$ is an $\beta$-cover with respect to $\ell_\infty$ for $\phi \circ \G$ on $\z$. Hence
\begin{align}
	\label{eq:upper_bd_on_cover_1}
\log \ \mathcal{N}_2(\beta, \phi \circ \G, \z )  \le \log(|V^\phi|) = \sum_{j=1}^k \log(|V_j|) = \sum_{j=1}^k \log\ \mathcal{N}_\infty(\beta, \G_j, \z ).
\end{align}
For any $g\in\G$ and $z\in\Z$, the value $\phi(g(z),z)$ is contained in the interval $[-1+\phi({\bf 0},z), +1+\phi({\bf 0},z)]$ by the Lipschitz property. Consider the $\reals$-valued tree $\phi({\bf 0},\cdot) \circ \z$. We now center by this tree and consider the set of trees $$\{\phi(g(\cdot),\cdot) \circ \z - \phi({\bf 0},\cdot) \circ \z: g\in\G\}$$ 
The centering does not change the size of the cover calculated in \eqref{eq:upper_bd_on_cover_1}, but allows us to invoke \eqref{eq:dudley} since the function values are now in $[-1,1]$:
\begin{align}
	\label{eq:dud_multi}
\Rad_T(\phi \circ \G, \z)  & \le  \inf_{\alpha}\left\{4 \alpha + \frac{12}{\sqrt{T}}\int_{\alpha}^{1} \sqrt{\sum_{j=1}^k \log \ \mathcal{N}_\infty(\beta,\G_j,\z ) \ } d \beta \right\} \notag\\
& \le  \inf_{\alpha}\left\{4 \alpha + \frac{12}{\sqrt{T}} \sum_{j=1}^k \int_{\alpha}^{1} \sqrt{\log \ \mathcal{N}_\infty(\beta,\G_j,\z ) \ } d \beta \right\}
\end{align}
We substitute the upper bound on covering numbers in \eqref{eq:fromcoveringtofat} for each $\G_j$ and arrive at an upper bound of 
\begin{align}
	\label{eq:intermed_dud_fat}
	\inf_{\alpha}\left\{4 \alpha + \frac{12}{\sqrt{T}} \sum_{j=1}^k \int_{\alpha}^{1} \sqrt{\fat_\beta(\G_j) \log (2eT/\beta) } d \beta \right\}.
\end{align}
Lemma 2 in \citep{RakSriTew13} implies that for any $\beta > 2 \Rad_T(\G_j)$,
$$
\fat_\beta(\G_j) \le \frac{32T\ \Rad_T(\G_j)^2}{\beta^2} \ .
$$
Let $j^*=\argmax{j} \Rad_T(\G_j)$. Substituting this together with the value of $\alpha = 2\Rad_T(\G_{j^*})$ into \eqref{eq:intermed_dud_fat} yields an upper bound
$$ 8\ \Rad_T(\G_{j^*}) + 48 \sqrt{2}\ \sum_{j=1}^k \Rad_T(\G_j) \int_{2 \Rad_T(\G_{j^*})}^1  \frac{1}{\beta}\sqrt{\log(2 e T / \beta)}  d \beta $$
Using the fact that for any $b>1$ and $\alpha\in(0,1)$
	\begin{align}
		\label{eq:bound_integral_dud}
		\int_\alpha^1 \frac{1}{\beta}\sqrt{\log (b/\beta)}d\beta = \int_{b}^{b/\alpha} \frac{1}{x}\sqrt{\log x}dx = \frac{2}{3}\log^{3/2} (x) \Big|_{b}^{b/\alpha} \leq \frac{2}{3} \log^{3/2} (b/\alpha) 
	\end{align}
we obtain a further upper bound of 
\begin{align*}
8\ \Rad_T(\G_{j^*})  + 32 \sqrt{2}\ \sum_{j=1}^k \Rad_T(\G_j) \ \log^{3/2}\left(\frac{e T}{\Rad_T(\G_{j^*})}\right)  \ .
\end{align*}
Replacing the first term by $8 \sum_{j}\Rad_T(\G_j)$, we conclude that
$$
\Rad_T(\phi \circ \G,\z)  \le 8\left(1 + 4\sqrt{2} \log^{3/2}(eT^2) \right) \sum_{j=1}^k \Rad_T(\G_j)
$$
as long as $\Rad_T(\G_j) \ge 1/T$ for each $j$. The statement is concluded by observing that $\z$ was chosen arbitrarily.
\end{proof}

\begin{proof}[\textbf{of Corollary~\ref{cor:radem_binary}}]
We first extend the binary function $b$ to a function $\bar{b}$ to any $x \in \reals^k$ as follows :
$$
\bar{b}(x) = \left\{\begin{array}{cl}
(1 - \|x - a\|_\infty)b(a) & \textrm{if }\|x - a\|_\infty < 1 \textrm{ for some }a \in \{\pm 1\}^k \\
0 & \textrm{otherwise} 
\end{array} \right.
$$ 
First note that $\bar{b}$ is well-defined since all points in the $k$-cube are separated by $L_\infty$ distance $2$. Further note that $\bar{b}$ is $1$-Lipschitz w.r.t. the $L_\infty$ norm and so applying Lemma \ref{lem:inflip} we conclude the statement of the corollary.
\end{proof}

\begin{proof}[\textbf{of Theorem~\ref{thm:valrad}}]
Let $\En_{t-1}[\cdot] = \En[\cdot | Z_1,\ldots,Z_{t-1}]$ denote the conditional expectation. Using Theorem~\ref{thm:minimax} we have,
\begin{align}
\Val_T(\F) & = \sup_{p_1} \Ex_{Z_1 \sim p_1} \ldots \sup_{p_T} \Ex_{Z_T \sim p_T}\left[ \sum_{t=1}^T  \inf_{f_t \in \F} \En_{t-1} \loss(f_t,\cdot) - \inf_{f \in \F} \sum_{t=1}^T \loss(f,Z_t) \right] \notag\\
\notag
& = \sup_{p_1} \Ex_{Z_1 \sim p_1} \ldots \sup_{p_T} \Ex_{Z_T \sim p_T}\left[ \sup_{f \in \F} \left\{ \sum_{t=1}^T  \inf_{f_t \in \F} \En_{t-1}  \loss(f_t,\cdot) -  \sum_{t=1}^T \loss(f,Z_t) \right\} \right] \notag\\
& \le \sup_{p_1} \Ex_{Z_1 \sim p_1} \ldots \sup_{p_T} \Ex_{Z_T \sim p_T}\left[ \sup_{f \in \F} \left\{ \sum_{t=1}^T \En_{t-1} \loss(f,\cdot) -  \sum_{t=1}^T \loss(f,Z_t) \right\} \right] \label{eq:F_subset_G}
\end{align}
The upper bound is obtained by replacing each infimum by a particular choice $f$. This step also holds if the choice $f_t$ of the learner comes from a larger set $\G$, as long as $\F\subseteq\G$. The proof is concluded by appealing to \eqref{eq:mainup}.

\end{proof}

\begin{proof}[\textbf{of Theorem \ref{thm:valrad_supervised}}]

Let $\tilde{Q}$ denote the set of distributions on $\Y=[-1,1]$. By convexity,
$$\sum_{t=1}^T \ell(\hat{y}_t,y_t) - \inf_{f\in \F}\sum_{t=1}^T \ell(f(x_t),y_t)\leq \sup_{f\in \F} \sum_{t=1}^T  \ell'(\hat{y}_t,y_t)\left(\hat{y}_t -  f(x_t) \right)$$	
where $\ell'(\hat{y}_t,y_t)$ is a subgradient of the function $y \mapsto \ell(\cdot, y_t)$ at $\hat{y}_t$. Then the minimax value \eqref{eq:sup_value} can be upper bounded as
\begin{align*}
 \Val^{S}_T(\F) \leq 
 \sup_{x_1} \inf_{q_1\in \tilde{Q}}\sup_{y_1} \Ex_{\hat{y}_1 \sim q_1} \ldots   \sup_{x_T} \inf_{q_T\in \tilde{Q}}\sup_{y_T}  \Es{\hat{y}_T \sim q_T}{  \sup_{f\in \F} \sum_{t=1}^T  \ell'(\hat{y}_t,y_t)\left(\hat{y}_t -  f(x_t) \right)}\notag 
\end{align*}
By the Lipschitz property of $\ell$, we can replace each subgradient $\ell'(\hat{y}_t,y_t)$ with a number $s_t \in [-L,L]$ to obtain the upper bound
\begin{align*}
 &\sup_{x_1} \inf_{q_1\in \tilde{Q}}\sup_{y_1} \Ex_{\hat{y}_1 \sim q_1} \sup_{s_1 \in [-L,L]} \ldots   \sup_{x_T} \inf_{q_T\in \tilde{Q}}\sup_{y_T}   \Ex_{\hat{y}_T \sim q_T} \sup_{s_T \in [-L,L]} \left\{\sup_{f \in \F} \sum_{t=1}^T s_t \left(\hat{y}_t -  f(x_t) \right)\right\}
\end{align*}
Since $y_t$'s no longer appear in the optimization objective, we can simply write the above as
\begin{align*}
 &\sup_{x_1} \inf_{q_1\in \tilde{Q}} \Ex_{\hat{y}_1 \sim q_1} \sup_{s_1 \in [-L,L]} \ldots   \sup_{x_T} \inf_{q_T\in \tilde{Q}}   \Ex_{\hat{y}_T \sim q_T} \sup_{s_T \in [-L,L]} \left\{\sup_{f \in \F} \sum_{t=1}^T s_t \left(\hat{y}_t -  f(x_t) \right)\right\}\notag \\
 & =  \sup_{x_1} \inf_{\hat{y}_1\in [-1,1]}  \sup_{s_1 \in [-L,L]} \ldots   \sup_{x_T} \inf_{\hat{y}_T\in [-1,1]}    \sup_{s_T \in [-L,L]} \left\{\sup_{f \in \F} \sum_{t=1}^T s_t \left(\hat{y}_t -  f(x_t) \right)\right\}
\end{align*}
where the equality follows because infima are obtained at point distributions. By the same reasoning, we now pass to distributions over $s_t$'s:
\begin{align}
\sup_{x_1} \inf_{\hat{y}_1\in [-1,1]}  \sup_{p_1  } \Ex_{s_1 \sim p_1} \ldots   \sup_{x_T} \inf_{\hat{y}_T\in [-1,1]}    \sup_{p_T  } \Es{s_T \sim p_T}{ \sum_{t=1}^T s_t \cdot \hat{y}_t - \inf_{f \in \F}   \sum_{t=1}^T s_t  f(x_t) } \label{eq:interapp}
\end{align}
From now on, it will be understood that the supremum over $p_t$ ranges over all distributions supported on $[-L,L]$, for any $t$, and $s_t$ has distribution $p_t$. Now note that 
$$\Es{s_T}{ \sum_{t=1}^T s_t \cdot \hat{y}_t - \inf_{f \in \F} \sum_{t=1}^T s_t \cdot f(x_t) }$$ is concave (linear) in $p_T$ and is convex in $\hat{y}_T$ and hence by the minimax theorem,
\begin{align*}
\inf_{\hat{y}_T\in [-1,1]}  & \sup_{p_T }  \Es{s_T}{ \sum_{t=1}^T s_t \cdot \hat{y}_t - \inf_{f \in \F} \sum_{t=1}^T s_t  f(x_t) }   =   \sup_{p_T}   \inf_{\hat{y}_T\in [-1,1]} \Es{s_T}{ \sum_{t=1}^T s_t \cdot \hat{y}_t - \inf_{f \in \F} \sum_{t=1}^T s_t  f(x_t)  }\\
&  =    \sum_{t=1}^{T-1}  s_t \cdot \hat{y}_t + \sup_{p_T}    \Es{s_T}{  \inf_{\hat{y}_T\in [-1,1]} \Es{s_T}{s_T} \cdot \hat{y}_T - \inf_{f \in \F} \sum_{t=1}^T s_t  f(x_t) }
\end{align*}
where the last step is similar to the one in the proof of Theorem \ref{thm:minimax}, specifically Eq.\eqref{eq:pullingstep}. Similarly note that the term 
$$\Es{s_{T-1}}{\sum_{t=1}^{T-1}  s_t \cdot \hat{y}_t + \sup_{p_T, x_T}    \Es{s_T}{  \inf_{\hat{y}_T\in [-1,1]} \Es{s_T}{s_T} \cdot \hat{y}_T - \inf_{f \in \F} \sum_{t=1}^T s_t  f(x_t) }}
$$ 
is  concave (linear) in $p_{T-1}$ and is convex in $\hat{y}_{T-1}$ and hence again by the minimax theorem,
\begin{align*}
&\inf_{\hat{y}_{T-1}\in [-1,1]}  \sup_{p_{T-1}}  \Ex_{s_{T-1}} \Bigg[ \sum_{t=1}^{T-1}  s_t \cdot \hat{y}_t  + \sup_{p_T,x_T}    \Ex_{s_T} \left[  \inf_{\hat{y}_T\in [-1,1]} \Es{s_T}{s_T} \cdot \hat{y}_T - \inf_{f \in \F} \sum_{t=1}^T s_t  f(x_t) \right] \Bigg]\\
&=  \sup_{p_{T-1}}  \inf_{\hat{y}_{T-1}\in [-1,1]}    \Ex_{s_{T-1}} \Bigg[ \sum_{t=1}^{T-1}  s_t \cdot \hat{y}_t + \sup_{p_T,x_T}    \Es{s_T}{  \inf_{\hat{y}_T\in [-1,1]} \Es{s_T}{s_T} \cdot \hat{y}_T - \inf_{f \in \F} \sum_{t=1}^T s_t  f(x_t) } \Bigg]\\
 =    &  \sum_{t=1}^{T-2}  s_t \cdot \hat{y}_t + \sup_{p_{T-1}} \Ex_{s_{T-1}} \sup_{p_{T},x_T} \Es{s_{T}}{ \sum_{t=T-1}^T   \inf_{\hat{y}_t\in [-1,1]} \Es{s_t}{s_t} \cdot \hat{y}_t - \inf_{f \in \F} \sum_{t=1}^T s_t  f(x_t) } 
\end{align*}
Proceeding in similar fashion and using this in Eq.\eqref{eq:interapp} we conclude that,
\begin{align*}
&\hspace{-0.1in} \Val^{S}_T(\F)  \le \sup_{x_1} \inf_{\hat{y}_1\in [-1,1]}  \sup_{p_1  } \Ex_{s_1 \sim p_1} \ldots   \sup_{x_T} \inf_{\hat{y}_T\in [-1,1]}    \sup_{p_T  } \Es{s_T \sim p_T}{ \sum_{t=1}^T s_t \cdot \hat{y}_t - \inf_{f \in \F}   \sum_{t=1}^T s_t  f(x_t) } \notag \\
& = \sup_{x_1}  \sup_{p_1  } \Ex_{s_1 \sim p_1} \ldots   \sup_{x_T}    \sup_{p_T  } \Ex_{s_T \sim p_T} \left[ \sum_{t=1}^T \inf_{\hat{y}_t \in [-1,1]} \Es{s_t \sim p_t}{s_t} \cdot \hat{y}_t - \inf_{f \in \F}   \sum_{t=1}^T s_t  f(x_t) \right]\notag \\
& \le \sup_{x_1}  \sup_{p_1  } \Ex_{s_1 \sim p_1} \ldots   \sup_{x_T}    \sup_{p_T  } \Es{s_T \sim p_T}{ \sup_{f \in \F} \sum_{t=1}^T \left( \Es{s_t \sim p_t}{s_t}  -     s_t \right) f(x_t) }\notag
\end{align*}
where we replaced each $\hat{y}_t$ with a potentially suboptimal choice $f(x_t)$. Passing the expectation past the suprema we obtain an upper bound
\begin{align}
&\sup_{x_1}  \sup_{p_1  } \Ex_{s_1, s'_1 \sim p_1} \ldots   \sup_{x_T}    \sup_{p_T  } \Es{s_T, s'_T \sim p_T}{ \sup_{f \in \F} \sum_{t=1}^T \left( s'_t  -     s_t \right) f(x_t) }\\
& = \sup_{x_1}  \sup_{p_1  } \Ex_{s_1, s'_1 \sim p_1}\Ex_{\epsilon_1} \ldots   \sup_{x_T}    \sup_{p_T  } \Ex_{s_T, s'_T \sim p_T}\Es{\epsilon_T}{ \sup_{f \in \F} \sum_{t=1}^T \epsilon_t \left( s'_t  -     s_t \right) f(x_t) }\notag \\
& \le \sup_{x_1}  \sup_{s_1 \in [-2L,2L]} \Ex_{\epsilon_1} \ldots   \sup_{x_T}    \sup_{s_T \in [-2L,2L]} \Es{\epsilon_T}{ \sup_{f \in \F} \sum_{t=1}^T \epsilon_t  s_t   f(x_t) }\notag \\
& = \sup_{x_1}  \sup_{s_1 \in \{-2L,2L\}} \Ex_{\epsilon_1} \ldots   \sup_{x_T}    \sup_{s_T \in \{-2L,2L\}} \Es{\epsilon_T}{ \sup_{f \in \F} \sum_{t=1}^T \epsilon_t  s_t   f(x_t) }\\
& = 2 L \ \sup_{x_1}  \sup_{s_1 \in \{-1,1\}} \Ex_{\epsilon_1} \ldots   \sup_{x_T}    \sup_{s_T \in \{-1,1\}} \Es{\epsilon_T}{ \sup_{f \in \F} \sum_{t=1}^T \epsilon_t  s_t   f(x_t) }\label{eq:inter1app}
\end{align}
where the last inequality is because, for every $t \in [T]$, we have convexity in $s_t$ and so supremum is achieved at either $-2L$ or $2L$. Notice that after using convexity to go to gradients, the proof technique above basically mimics the proofs of Theorems \ref{thm:minimax} and \ref{thm:valrad} to get to a symmetrized term as we did in those theorems. Now consider any arbitrary function $\psi : \{\pm 1\} \mapsto \reals$, we have that 
$$
\sup_{s \in \{\pm 1\}} \Es{\epsilon}{\psi(s\cdot\epsilon)} = \sup_{s \in \{\pm 1\}} \frac{1}{2}\left(\psi(+s) + \psi(-s)\right) =  \frac{1}{2} \left(\psi(+1) + \psi(-1)\right) = \Es{\epsilon}{\psi(\epsilon)}
$$
Since in Eq.~\eqref{eq:inter1app}, for each $t$, $s_t$ and $\epsilon_t$ appear together as $\epsilon_t \cdot s_t$ using the above equation repeatedly, we conclude that 
\begin{align}
	\label{eq:sup_value_upper_1}
\Val^{S}_T(\F)  & \le 2 L \ \sup_{x_1}  \sup_{s_1 \in \{-1,1\}} \En_{\epsilon_1} \ldots   \sup_{x_T}    \sup_{s_T \in \{-1,1\}} \Es{\epsilon_T}{ \sup_{f \in \F} \sum_{t=1}^T \epsilon_t  s_t   f(x_t) } \notag\\
& = 2 L \ \sup_{x_1}  \En_{\epsilon_1} \ldots   \sup_{x_T}    \Es{\epsilon_T}{ \sup_{f \in \F} \sum_{t=1}^T \epsilon_t    f(x_t) }
\end{align}
We now claim that the above supremum can be written in terms of an $\X$-valued tree. Briefly, the solution for $x_1$ in \eqref{eq:sup_value_upper_1} is attained (for simplicity, assume the supremum is attained) at an optimal value $x_1^*$. The optimal value $x_2^*$ can be calculated for $\epsilon_1=1$ and $\epsilon_1=-1$. Arguing in this manner leads to a tree $\x$. We conclude 
\begin{align*}
\Val^{S}_T(\F)  &\leq 2 L\ \sup_{\x} \Es{\epsilon_{1:T}}{ \sup_{f \in \F} \sum_{t=1}^T \epsilon_t f(\x_t(\epsilon)) } = 2\, L\, T\ \Rad_T(\F)
\end{align*}

\end{proof}

\begin{proof}[\textbf{of Proposition~\ref{prop:uplow}}]
For the upper bound, we start by using Theorem~\ref{thm:valrad_supervised} for absolute loss, which has a Lipschitz constant of $1$,
to bound the value of the game by sequential Rademacher complexity,
\[
	\frac{1}{T} \Val^{\trm{S}}_T(\F) \le 2\, \Rad_T(\F) \ .
\]
We combine the above inequality with Eqs.~\eqref{eq:dudley} and~\eqref{eq:fromcoveringtofat} to obtain the upper bound.

Observe that a lower bound on the value can be obtained by choosing any particular joint distribution on sequences $(x_1,y_1),\ldots,(x_t,y_t)$ in Eq.~\eqref{eq:value_equality}:
\begin{align*}
	\Val^{\trm{S}}_T(\F) &\geq  \En \left[
	  \sum_{t=1}^T \inf_{f_t \in \F}
	  \Es{(x_t,y_t)}{ |y_t-f_t (x_t)| \ \Big|\  (x,y)_{1:t-1}} - \inf_{f\in \F} \sum_{t=1}^T |y_t-f(x_t)|
	\right] 
\end{align*}
To this end, choose any $\X$-valued tree $\x$ of depth $T$. Let $y_1,\ldots,y_T$ be i.i.d. Rademacher random variables and define $x_t = \x(y_{1:t-1})$ deterministically (that is, the conditional distribution of $x_t$ is a point distribution on $\x(y_{1:t-1})$). It is easy to see that this distribution makes the choice $f_t$ irrelevant, yielding
\begin{align*}
	\Val^{\trm{S}}_T(\F) &\geq \En \left[
	  \sum_{t=1}^T 1 - \inf_{f\in \F} \sum_{t=1}^T |y_t-f(x_t)|
	\right] 
	= \En_{y_1,\ldots,y_T} 
	   \sup_{f\in \F} \sum_{t=1}^T y_t f(x_t)
\end{align*}
Since this holds for any tree $\x$, we obtain the desired lower bound $\Val^{\trm{S}}_T(\F) \geq \Rad_T(\F)$. The final lower bound on $\Rad_T(\F)$ (in terms of the fat-shattering dimensions) is proved by \citet[Lemma 2]{RakSriTew13}.

\end{proof}

\begin{proof}[\textbf{of Theorem~\ref{thm:tight}}]	
	The equivalence of \emph{1} and \emph{2} follows directly from Proposition~\ref{prop:uplow}. First, suppose that $\fat_{\alpha}$ is infinite
	for some $\alpha > 0$. Then, the lower bound says that $\Val^\trm{S}_T(\F) \ge \alpha T/(4\sqrt{2})$ and hence
	$\limsup_{T\to\infty} \Val^\trm{S}_T(\F)/T \ge \alpha/(4\sqrt{2})$. Thus, the class $\F$ is not online learnable in the supervised setting. Now, assume that $\fat_\alpha$
	is finite for all $\alpha$. Fix an $\epsilon > 0$ and choose $\alpha = \epsilon/16$. Using the upper bound, we have
	\begin{align*}
	\Val^\trm{S}_T(\F) &\le 8T\alpha + 24\sqrt{T} \int_{\alpha}^1 \sqrt{ \fat_\beta \log\left(\frac{2 e T}{\beta}\right)}\ d \beta \\
	&\le 8T\alpha + 24\sqrt{T} (1-\alpha) \sqrt{ \fat_\alpha \log\left(\frac{2 e T}{\alpha}\right) } \\
	&\le \epsilon T/2 + \epsilon T/2
	\end{align*}
	for $T$ large enough. Thus, $\limsup_{T\to\infty} \Val^{\trm{S}}_T(\F)/T \le \epsilon$. Since $\epsilon > 0$ was arbitrary, this proves that
	$\F$ is online learnable in the supervised setting.

	The statement that  $\Val^{\trm{S}}_T(\F)$, $\Rad_T(\F)$, and $\Dudley_T(\F)$ are within a multiplicative factor of $\mc{O}(\log^{3/2} T)$ of each other whenever the problem is online learnable follows immediately from  \cite[Eq.~(10)]{RakSriTew13} and Proposition~\ref{prop:uplow}.

\end{proof}

\begin{proof}[\textbf{of Lemma~\ref{lem:equal_value}}]
Consider the game $(\F,\Zcvx)$ and fix a randomized strategy $\pi$ of the player.
Then, the expected regret of a randomized strategy $\pi$ against any adversary playing $g_1,\ldots,g_T$ can be lower-bounded via Jensen's inequality as
\begin{align*}
\sum_{t=1}^T \Es{u_t \sim \pi_t(g_{1:t-1})}{ g_t(u_t) } - \inf_{u \in \F} \sum_{t=1}^T g_t(u) \ge \sum_{t=1}^T g_t\left(\Es{u_t \sim \pi_t(g_{1:t-1})}{u_t}\right) - \inf_{u \in \F} \sum_{t=1}^T g_t(u),
\end{align*}
which is simply regret of a {\em deterministic} strategy obtained from $\pi$ by playing $\Es{u_t \sim \pi_t(g_{1:t-1})}{u_t}$ 
on round $t$.
Thus, to any randomized strategy corresponds a deterministic one that is no worse. On the other hand, the set of randomized strategies contains the set of deterministic ones. Hence, $\Val_T(\F,\Zcvx) = \Valdet_T(\F,\Zcvx)$ where $\Valdet_T$ is defined as the minimax regret obtainable only using deterministic player
strategies. Now, we appeal to Theorem 14 of \citet{abernethy08optimal} that says $\Valdet_T(\F,\Zcvx) = \Valdet_T(\F,\Zlin)$. Note that \citet{abernethy08optimal} deal with convex sets in finite dimensional spaces only. However, their proof relies on fundamental properties of convex functions that are true in any general vector space (such as the fact that the first order Taylor expansion of a convex function globally lower bounds the convex function). Since $\Zlin$ also consists
of convex (in fact, linear) functions, the above argument again gives $\Valdet_T(\F,\Zlin) = \Val_T(\F,\Zlin)$. This finishes the proof of the lemma.
\end{proof}

\begin{proof}[\textbf{of Proposition~\ref{prop:NN}}]
We shall prove that for any $i \in \{2,\ldots,k\}$,
$$
\Rad_T(\F_i) \le 16 L B_i \left(1+4\sqrt{2}\log^{3/2}(eT^2)\right) \Rad_T(\F_{i-1})
$$
To see this note that for any $\x$, $\Rad_T(\F_i,\x)$ is equal to
\begin{align*}
\Es{\epsilon}{\sup_{\underset{ \forall j~  f_j \in \F_{i-1}}{w^{i} : \|w^i\|_1 \le B_i}}\sum_{t=1}^T \epsilon_t \left(\sum_{j} w^i_j \sigma\left(f_j(\x_t(\epsilon))\right)\right) }
\le \Es{\epsilon}{\sup_{\underset{ \forall j~  f_j \in \F_{i-1}}{w^{i} : \|w^i\|_1 \le B_i}} \|w^i\|_1 \max_j\left|\sum_{t=1}^T \epsilon_t \sigma\left(f_j(\x_t(\epsilon))\right)\right|} 
\end{align*}
by H\"older's inequality. Then $\Rad_T(\F_i)$ is upper bounded as
\begin{align}
&\sup_{\x} \Es{\epsilon}{ B_i \sup_{ f \in \F_{i-1}}\max\left\{\sum_{t=1}^T \epsilon_t \sigma\left(f(\x_t(\epsilon))\right), -\sum_{t=1}^T \epsilon_t \sigma\left(f(\x_t(\epsilon))\right)\right\}}\notag \\
&\leq \sup_{\x} \Es{\epsilon}{ B_i \max\left\{\sup_{ f \in \F_{i-1}}\sum_{t=1}^T \epsilon_t \sigma\left(f(\x_t(\epsilon))\right), \sup_{ f \in \F_{i-1}} \sum_{t=1}^T - \epsilon_t \sigma\left(f(\x_t(\epsilon))\right)\right\}}\notag \ .
\end{align}
Since $0 \in \F_i$ together with the assumption of $\sigma(0) =0$, both terms are non-negative, and thus the maximum above can be upper bounded by the sum
\begin{align}
& \sup_{\x} \Es{\epsilon}{ B_i \sup_{ f \in \F_{i-1}} \sum_{t=1}^T \epsilon_t \sigma\left(f(\x_t(\epsilon))\right)} + \sup_{\x} \Es{\epsilon}{ B_i \sup_{ f \in \F_{i-1}} \sum_{t=1}^T - \epsilon_t \sigma\left(f(\x_t(\epsilon))\right)} \notag \ .
\end{align}
We now claim that the two terms are equal. Indeed, let $\x^*$ be the tree achieving the supremum in the first term (a modified analysis can be carried out if the supremum is not achieved). Then the mirror tree $\x$ defined via $\x_t(\epsilon)=\x^*_t(-\epsilon)$ yields the same value for the second term. Since the argument can be carried out in the reverse direction, the two terms are equal, and the upper bound of
\begin{align}
&2 B_i  \sup_{\x} \Es{\epsilon}{  \sup_{ f \in \F_{i-1}}\sum_{t=1}^T \epsilon_t \sigma\left(f(\x_t(\epsilon))\right)} \notag
\end{align}
follows. In view of contraction in Corollary~\ref{cor:contraction}, we obtain a further upper bound of
\begin{align}
	&16 B_i L \left(1+4\sqrt{2}\log^{3/2}(eT^2)\right) \Rad_T(\F_{i-1}) \label{eq:rec}
\end{align}
To finish the proof we note that for the base case of $i=1$, $\Rad_T(\F_1)$ is equal to
\begin{align*}
\sup_{\x} \Es{\epsilon}{  \sup_{ w \in \reals^d : \|w\|_1 \le B_1}\sum_{t=1}^T \epsilon_t w^\top \x_t(\epsilon)} 
\end{align*}
which is upper bounded by
\begin{align*}
\sup_{\x} \Es{\epsilon}{  \sup_{ w \in \reals^d : \|w\|_1 \le B_1}\|w\|_1 \left\|\sum_{t=1}^T \epsilon_t \x_t(\epsilon)\right\|_\infty} \leq B_1 \sup_{\x} \Es{\epsilon}{ \max_{i\in [d]}\left\{\sum_{t=1}^T \epsilon_t \x_t(\epsilon)[i] \right\}}
\end{align*}
Note that the instances $x \in \X$ are vectors in $\reals^d$ and so for a given instance tree $\x$, for any $i \in [d]$, $\x[i]$ given by only taking the $i^{th}$ co-ordinate is a valid real valued tree. By Eq.~\eqref{eq:fin},
\begin{align*}
T \cdot \Rad_T(\F_1) & \le B_1 \sup_{\x} \Es{\epsilon}{ \max_{i\in [d]}\left\{\sum_{t=1}^T \epsilon_t \x_t(\epsilon)[i]\right\}}  \le B_1 \sqrt{2 T X_\infty^2 \log d}
\end{align*}
Using the above and Eq.~\eqref{eq:rec} repeatedly we conclude the proof.
\end{proof}

\begin{proof}[\textbf{of Proposition~\ref{prop:margin}}]
Fix a $\gamma > 0$ and use loss 
$$
\ell(\hat{y},y) = \left\{\begin{array}{ll}
1 &  \hat{y} y \le 0\\
1-\hat{y}y/\gamma & 0 < \hat{y} y < \gamma \\
0 & \hat{y} y \ge \gamma
\end{array}\right.
$$
Since this loss is $1/\gamma$-Lipschitz, we can use \eqref{eq:sup_value_upper_bound_no_contraction} and the Rademacher contraction Corollary~\ref{cor:contraction} to show that for each $\gamma > 0$ there exists a randomized strategy $\tau^\gamma$ such that for any data sequence
$$
\sum_{t=1}^T  \Es{\hat{y}_t \sim \tau^\gamma_t(z_{1:t-1})}{\ell(\hat{y}_t, y_t)}  \le \inf_{f \in \F} \sum_{t=1}^T \ell(f(x_t), y_t) +  \gamma^{-1} \rho_T T\Rad_T(\F), 
$$
where $\rho_T = 16\left(1+4\sqrt{2}\log^{3/2}(eT^2)\right)$ throughout the proof. Further, observe that the loss function is lower bounded by the zero-one loss $\ind{\hat{y} y <0}$ and is upper bounded by the margin zero-one loss $\ind{\hat{y} y <\gamma}$. Hence, 
\begin{align} \label{eq:gamreg}
\sum_{t=1}^T  \Es{\hat{y}_t \sim \tau^\gamma_t(z_{1:t-1})}{\ind{\hat{y}_t y_t < 0}}  \le \inf_{f \in \F} \sum_{t=1}^T \ind{y_tf(x_t) < \gamma} +  \gamma^{-1}  \rho_T T\Rad_T(\F) 
\end{align}

The above bound holds for randomized each strategy given by $\tau^\gamma$, for any given $\gamma$. Now we discretize the set of  $\gamma$'s as $\gamma_i = 1 /2^{i}$ and use the output of the randomized strategies $\tau^{\gamma_1}, \tau^{\gamma_2}, \ldots$, that attain the regret bounds given in \eqref{eq:gamreg}, as experts. We then run a countable experts algorithm (Algorithm~\ref{alg:experts}) with initial weight for expert $i$ as $p_i = \frac{6}{\pi^2 i^2}$. Such an algorithm achieves $\mathcal{O}(\sqrt{T} \log (1/p_i))$ regret w.r.t. expert $i$. In view of Proposition \ref{prop:experts}, for this randomized strategy $\tau$, for any $i$
\begin{align*}
\sum_{t=1}^T  \Es{\hat{y}_t \sim \tau_t(z_{1:t-1})}{\ind{\hat{y}_t y_t < 0}}  \le \inf_{f \in \F} \sum_{t=1}^T \ind{y_tf(x_t) < \gamma_i} +  \gamma_i^{-1} \rho_T T\Rad_T(\F)  +\sqrt{T}\left(1 + 2 \log\left(\frac{i \pi}{\sqrt{6}}\right)\right)
\end{align*}
For any $\gamma > 0$, let $i_\gamma\in 0,1,\ldots,$ be such that $2^{-(i_\gamma+1)}< \gamma \le 2^{-i_\gamma}$. Then above right-hand side is upper bounded by
\begin{align*}
\inf_{f \in \F} \sum_{t=1}^T \ind{y_tf(x_t) < 2 \gamma} 
+  \gamma^{-1}\rho_T T\Rad_T(\F)  + \sqrt{T}\left(1 + 2 \log\left(\frac{i_\gamma \pi}{\sqrt{6}}\right)\right)
\end{align*}
The proof is concluded using the inequality $i_\gamma \le \log(1/\gamma)$ and upper bounding constants.
\end{proof}

\begin{proof}[\textbf{of Proposition~\ref{prop:DT}}]
Fix some $L > 0$. The loss 
$$
\phi_L(\alpha) = \left\{\begin{array}{cl}
1 & \textrm{if }\alpha \le 0\\
1 - L\alpha & \textrm{if }0 < \alpha \le 1/L\\
0 & \textrm{otherwise}
\end{array} \right.
$$ 
is $L$-Lipschitz and so by Theorem \ref{thm:valrad} and Corollary~\ref{cor:contraction} we have that for every $L > 0$, there exists a randomized strategy $\tau^L$ for the player, such that for any sequence $z_1 = (x_1,y_1), \ldots, z_T = (x_T,y_T)$,
\begin{align}
	\label{eq:dt_1}
\sum_{t=1}^T \Es{\hat{y}_t \sim \tau^L_t(z_{1:t-1})}{\phi_L(y_t\hat{y}_t)} \le \inf_{f \in \F} \sum_{t=1}^T \phi_L(y_t f(x_t)) + L \rho_T T\Rad_T(\F)
\end{align}
where $\rho_T = 16\left(1+4\sqrt{2}\log^{3/2}(eT^2)\right)$ throughout this proof. Since $\phi_L$ dominates the step function, the left hand side of \eqref{eq:dt_1} also upper-bounds the expected indicator loss
$$
\sum_{t=1}^T \Es{\hat{y}_t \sim \tau^L_t(z_{1:t-1})}{\ind{\hat{y}_t \ne y_t}}. 
$$
For any $f\in\F$, we can relate the $\phi_L$-loss to the indicator loss by
\begin{align*}
\sum_{t=1}^T \phi_L(y_t f(x_t))  &= \sum_{t=1}^T \ind{y_t f(x_t) \leq 0} + \sum_l C(l) \phi_L(w_l). 
\end{align*}
Let us now use the above decomposition in Eq.~\eqref{eq:dt_1}. Crucially, the sign of $f(x)$ does not depend on $w_l$, but only on the label $\sigma_l$ of the unique leaf $l$ reached by $x$. Thus, the infimum in \eqref{eq:dt_1} can be split into two infima: 
$$\inf_{f \in \F} \sum_{t=1}^T \phi_L(y_t f(x_t)) = \inf_{f\in\F} \sum_{t=1}^T \ind{y_t f(x_t) \leq 0} + \inf_{w_l} \sum_l C(l) \phi_L(w_l) $$
where it is understood that the $C(l)$ term on the right hand side is computed using the function $f$ minimizing the first sum on the right hand side.
We can further write
$$ \sum_l C(l) \phi_L(w_l) \le \sum_l C(l) \max(0, 1 - Lw_l) = \sum_l  \max\left(0, (1 - Lw_l)C(l) \right).$$
So far, we have derived a regret bound for a given $L$. Let us now remove the requirement to know $L$ a priori by running the experts Algorithm \ref{alg:experts} with $\tau^{1}, \tau^{2}, \ldots$ as a countable set of experts corresponding to the values $L \in \mbb{N}$. The prior on expert $L$ is taken to be $p_L = \frac{6}{\pi^2}L^{-2}$  so that $\sum p_L = 1$. For the randomized strategy $\tau$ obtained in this manner, from Proposition \ref{prop:experts}, for any sequence of instances and any $L \in \mbb{N}$, 
\begin{align*}
\sum_{t=1}^T \Es{\hat{y}_t \sim \tau_t(z_{1:t-1})}{\ind{\hat{y} \ne y_t}} &\le \inf_{f \in \F} \sum_{t=1}^T \ind{y_tf(x_t) \leq 0} + \inf_{f\in\F} \sum_l  \max\left(0, (1 - Lw_l)C(l) \right) \\
&+ L \rho_T T \Rad_T(\F)  + \sqrt{T} + 2 \sqrt{T} \log(L \pi / \sqrt{6})
\end{align*}
Now we pick $L = \left|\{l : C(l) > \rho_T T \Rad_T(\F)\}\right| \leq N$ and upper bound the second infimum by choosing $w_l = 0$ if $C(l) \le \rho_T T \Rad_T(\F)$ and $w_l = 1/L$ otherwise:
\begin{align*}
	\inf_{w_l}\sum_l  \max\left(0, (1 - Lw_l)C(l) \right) + L \rho_T T \Rad_T(\F) &\leq  \sum_l  C(l) \ind{C(l) \le  \rho_T T \Rad_T(\F)} \\
	&+ \rho_T T\Rad_T(\F) \sum_l \ind{C(l) > \rho_T T \Rad_T(\F)}
\end{align*}
which can be written succinctly as
$$\sum_l  \min\{C(l), \rho_T T \Rad_T(\F)\} $$
We conclude that
\begin{align*}
\sum_{t=1}^T \Es{\hat{y}_t \sim \tau_t(z_{1:t-1})}{\ind{\hat{y}_t \neq y_t}} &\leq \inf_{f\in\F} \sum_{t=1}^T \ind{y_t f(x_t)\leq 0} \\
&+ \sum_l \min(C(l), \rho_T T\Rad_T(\F)) + \sqrt{T}\left(1 + 2   \log(N \pi / \sqrt{6})\right)
\end{align*}
Finally, we apply Corollary \ref{cor:radem_binary} and Lemma~\ref{lem:rad_properties}(2) to bound $\Rad_T(\F) \le d \mc{O}(\log^{3/2}T)\  \Rad_T(\mc{H})$ and thus conclude the proof.
\end{proof}

\begin{proof}[{\bf of Proposition~\ref{prop:isotron}}]
First, by the classical result of \citet{kolmogorov1959}, the class $\G$ of all bounded Lipschitz functions on a bounded interval has small metric entropy: $\log\Nhat_\infty(\alpha, \G) = \Theta(1/\alpha)$. For the particular class of non-decreasing $1$-Lipschitz functions, it is trivial to verify that the entropy is in fact bounded by $2/\alpha$. Considering all $1$-Lipschitz functions increases this to $c_0/\alpha$ for some universal constant $c_0$.

Next, consider the class $\F = \{ \inner{w, x} \ | \ \|w\|_2\leq 1 \}$ over the Euclidean ball. By Proposition~\ref{prop:rad_linear_functions}, $\Rad_T(\F) \leq 1/\sqrt{T}$. Using the lower bound of Proposition~\ref{prop:uplow}, $\fat_\alpha \leq 32/\alpha^2$ whenever $\alpha > 4\sqrt{2}/\sqrt{T}$. This implies that $\N_\infty(\alpha,\F,T)\leq (2eT/\alpha)^{32/\alpha^2}$ whenever $\alpha > 4\sqrt{2}/\sqrt{T}$. Note that this bound does not depend on the ambient dimension of $\X$.

Next, we show that a composition of $\G$ with any ``small'' class $\F\subset [-1,1]^\X$ also has a small cover. To this end, suppose $\N_\infty (\alpha, \F, T)$ is the covering number for $\F$. Fix a particular tree $\x$ and let $V=\{\v_1,\ldots, \v_N\}$ be an $\ell_\infty$ cover of $\F$ on $\x$ at scale $\alpha$. Analogously, let $W=\{g_1,\ldots,g_M\}$ be an $\ell_\infty$ cover of $\G$ with $M = \Nhat_\infty(\alpha, \G)$. Consider the class $\G\circ \F = \{g\circ f: g\in \G, f\in \F\}$. The claim is that $\{g(\v): \v\in V, g\in W\}$ provides an $\ell_\infty$ cover for $\G\circ\F$ on $\x$. Fix any $f\in\F, g\in \G$ and $\epsilon\in\{\pm1\}^T$. Let $\v\in V$ be such that $\max_{t\in[T]} |f(\x_t(\epsilon))-\v_t(\epsilon)| \leq \alpha$, and let $g'\in W$ be such that $\|g-g'\|_\infty\leq\alpha$. Then, using the fact that functions in $\G$ are $1$-Lipschitz, for any $t\in [T]$, 
$$|g(f(\x_t(\epsilon))) - g'(\v_t(\epsilon))| \leq |g(f(\x_t(\epsilon))) - g'(f(\x_t(\epsilon))| + |g'(f(\x_t(\epsilon)) - g'(\v_t(\epsilon))| \leq 2\alpha \ .$$
Hence, $\N_\infty(2\alpha, \G\circ\F, T) \leq   \Nhat_\infty (\alpha, \G) \times \N_\infty (\alpha, \F, T)$. 

Finally, we put all the pieces together. By Theorem~\ref{thm:valrad_supervised}, the minimax value is bounded by $8T$ times the sequential Rademacher complexity of the class
$\G\circ\F = \{ u(\inner{w, x}) \ | \ u:[-1,1]\to [-1,1] \mbox{ is $1$-Lipschitz }, \ \|w\|_2\leq 1 \} $
since the squared loss is $4$-Lipschitz on the space of possible values. The latter complexity is then bounded by
\begin{align*}
	T \Dudley_T(\G\circ\F) &\leq 32\sqrt{T} + 12\int_{8/\sqrt{T}}^{1} \sqrt{T \ \log \ \N(\delta, \G\circ\F, T) \ } d \delta \\
	&\leq 32\sqrt{T}+ 12\sqrt{T}\int_{8/\sqrt{T}}^1 \sqrt{\frac{4c_0}{\delta} + \frac{128}{\delta^2}\log(2eT)} d\delta \ .
\end{align*}
We therefore conclude that the value of the game for the supervised learning problem is bounded by $\mc{O}(\sqrt{T}\log^{3/2} (T))$.
\end{proof}

\section{Exponentially Weighted Average (EWA) Algorithm on Countable Experts}
We consider here a version of the exponentially weighted experts algorithm for a countable (possibly infinite) number of experts and provide a bound on the expected regret of the randomized algorithm. The proof of the result closely follows the finite case (e.g. \cite[Theorem 2.2]{PLG}). This result is well known and we include it here for completeness, as it is needed in the proofs of Proposition~\ref{prop:margin} and Proposition~\ref{prop:DT}. 

Suppose we are provided with countable experts $E_1, E_2 , \ldots$, where each expert can herself be thought of as a randomized/deterministic player strategy which, given history, produces an element of $\F$ at round $t$. Here we also assume that $\F \subseteq [0,1]^\X$. Denote by $f^i_t$ the function output by expert $i$ at round $t$ given the history. The EWA algorithm we consider needs access to the countable set of experts and also needs an initial weighting on each expert $p_1,p_2,\ldots$ such that $\sum_i p_i =1$. 

\begin{algorithm}
\caption{EWA ($E_1,E_2,\ldots$, $p_1,p_2,\ldots$)}
\label{alg:experts}
\begin{algorithmic}
\STATE Initialize each $w^1_i \gets p_i$
\FOR{$t=1$ to $T$}
\STATE Pick randomly an expert $i$ with probability $w^t_i$
\STATE Play $f_t = f^t_i$
\STATE Receive $x_t$
\STATE Update for each $i$, $w^{t+1}_i = \frac{w^t_i e^{- \eta f^t_i(x_t)}}{\sum_{i} w^t_i e^{- \eta f^t_i(x_t)}}$
\ENDFOR 
\end{algorithmic}
\end{algorithm}

\begin{proposition}
	\label{prop:experts}
The exponentially weighted average forecaster (Algorithm~\ref{alg:experts}) with $\eta = T^{-1/2}$ enjoys the regret bound
$$
\sum_{t=1}^T \E{f_t(x_t)} \le \sum_{t=1}^T f^t_i(x_t) + \frac{\sqrt{T}}{8} + \sqrt{T} \log\left( 1/p_i\right)
$$
for any $i \in \mathbb{N}$.
\end{proposition}